\documentclass[twoside,11pt]{article}
\usepackage{jmlr2e}
\usepackage{dsfont}
\usepackage{float}
\ShortHeadings{Learning to Benchmark}{Noshad, Xu and Hero}
\firstpageno{1}

\usepackage{microtype}
\usepackage{booktabs} 
\usepackage{url}

\usepackage{hyperref}

\usepackage{amsmath}
\usepackage{amsfonts}
\usepackage{graphicx}
\usepackage{stmaryrd}
\usepackage{bm}
\usepackage{array}
\usepackage{amstext}
\usepackage{latexsym}
\usepackage{url}
\usepackage{setspace}
\usepackage{rerunfilecheck}
\usepackage{xparse}
\usepackage{color}

\usepackage{subfigure}
\usepackage[ruled,vlined,linesnumbered]{algorithm2e}

\DeclareMathOperator*{\argmax}{arg\,max}

\newtheorem{assumption}[theorem]{Assumption}

\newtheorem{tab}{Table}

\newcommand{\eps}{\varepsilon}
\newcommand{\epsBayes}{\mathcal{E}^{\textrm{Bayes}}}

\newcommand{\CBayes}{C^{\textrm{Bayes}}}

\newcommand{\bp}{\bm{\mathbf{p}}}

\newcommand{\bw}{\bm{\mathbf{w}}}

\newcommand{\bS}{\bm{\mathbf{S}}}

\newcommand{\bX}{\bm{\mathbf{X}}}
\newcommand{\bY}{\bm{\mathbf{Y}}}
\newcommand{\bZ}{\bm{\mathbf{Z}}}

\newcommand{\hatp}{\hat{p}}

\newcommand{\Best}{\widehat{\mathcal{E}}_\eps(\bX_1,\bX_2)}
\DeclareMathSizes{10}{9}{8}{7}

\begin{document}

\title{Learning to Benchmark: Determining Best Achievable Misclassification Error from Training Data}

\author{\name Morteza Noshad \email noshad@umich.edu \newline
       \addr Department of Electrical Engineering and Computer Science\newline
       University of Michigan\newline
       Ann Arbor, MI 48109, USA
       \AND
       \name Li Xu \email lixu@ict.ac.cn \newline
       \addr
       (Corresponding Author)\newline
       Institute of Computing Technology\newline
       Chinese Academy of Sciences\newline
       Beijing 100190, China
       \AND
       \name Alfred Hero \email hero@eecs.umich.edu \newline
       \addr Department of Electrical Engineering and Computer Science\newline
       University of Michigan\newline
       Ann Arbor, MI 48109, USA}


\maketitle

\begin{abstract}
We address the problem of learning to benchmark the best achievable classifier performance. In this problem the objective is to establish statistically consistent estimates of the Bayes misclassification error rate without having to learn a Bayes-optimal classifier. Our learning to benchmark framework improves on previous work on learning bounds on Bayes misclassification rate since it learns the {\it exact} Bayes error rate instead of a bound on error rate.  We propose a benchmark learner based on an ensemble of $\eps$-ball estimators and Chebyshev approximation. Under a smoothness assumption on the class densities we show that our estimator achieves an optimal (parametric) mean squared error (MSE) rate of $O(N^{-1})$, where $N$ is the number of samples.

Experiments on both simulated and real datasets establish that our proposed benchmark learning algorithm produces estimates of the Bayes error that are more accurate than previous approaches for learning bounds on Bayes error probability.
\end{abstract}


\section{Introduction}
This paper proposes a framework for empirical estimation  of minimal achievable classification error, i.e., Bayes error rate, directly from training data, a framework we call {\em learning to benchmark}. Consider an observation-label pair $(X,T)$ takes values in $\mathbb{R}^d \times \{1,2,\ldots,\lambda \}$. For class $i$, the prior probability is $\Pr\{T=i\}=p_i$ and $f_i$ is the conditional distribution function of $X$ given that $T=i$. Let $\bp=(p_1,p_2,\ldots,p_\lambda )$. A classifier $C:\mathbb{R}^d \to \{1,2,\ldots,\lambda \}$ maps each $d$-dimensional observation vector $X$ into one of $\lambda $ classes. The misclassification error rate of $C$ is defined as
\begin{equation}
\mathcal{E}_C=\Pr(C(X)\ne T),
\end{equation}
which is the probability of classification associated with classifier function $C$. Among all possible classifiers, the Bayes classifier achieves minimal misclassification rate and has the form of a maximum {a posteriori} (MAP) classifier:
\begin{equation}
\CBayes(x)=\argmax_{1\le i\le \lambda } \Pr(T=i|X=x),
\end{equation}
The Bayes misclassification error rate is
\begin{equation}
\epsBayes_{\bp}(f_1,f_2,\ldots,f_\lambda )=\Pr(\CBayes(X)\ne T).
\label{eq:Bayes_error}
\end{equation}

The problem of learning to bound the Bayes error probability (\ref{eq:Bayes_error}) has generated much recent interest \citep{wang2005divergence}, \citep{poczos2012nonparametric}, \citep{berisha2016empirically},\citep{noshad2018hash}, \citep{moon2018ensemble}.
Approaches to this problem have proceeded in two stages: 1) specification of lower and upper bounds that are functions of the class probabilities (priors) and the class-conditioned distributions (likelihoods); and 2) specification of good empirical estimators of these bounds given a data sample.  The class of $f$-divergences \citep{ali1966general}, which are measures of dissimilarity between a pair of distributions, has been a fruitful source of bounds on the Bayes error probability and include: 
the Kullback-Leibler (KL) divergence \citep{kullback1951information}, the R\'enyi divergence \citep{renyi1961measures} the Bhattacharyya (BC) divergence \citep{bhattacharyya1946measure},  Lin's divergences \citep{lin1991divergence}, and the Henze-Penrose (HP) divergence \citep{henze1999multivariate}. For example, the HP divergence

\begin{align}\label{HP_divergence}
&D_{\bp}(f_1,f_2):=\frac{1}{4p_1p_2}\left[\int \frac{(p_1 f_1(x)-p_2 f_2(x))^2}{p_1 f_1(x)+p_2 f_2(x)}dx-(p_1-p_2)^2\right].
\end{align}
provides the bounds \citep{berisha2016empirically}:
\begin{align}\label{HP_bound}
&\frac{1}{2}-\sqrt{4p_1p_2D_{\bp}(f_1,f_2)+(p_1-p_2)^2}\le\epsBayes_{\bp}(f_1,f_2)\le 2p_1p_2(1-D_{\bp}(f_1,f_2)).
\end{align}
A consistent empirical estimator of the HP divergence (\ref{HP_divergence}) was given in \citep{friedman}, and this was used to learn the bounds (\ref{HP_bound}) in \citep{berisha2016empirically}.  Many alternatives to the HP  divergence have been used to solve the learning to bound problem including the Fisher Information \citep{berisha2014empirical}, the Bhattacharrya divergence \citep{berisha2016empirically}, the R\'enyi divergence \citep{noshad2018hash}, and the Kullback-Liebler divergence \citep{poczos2012nonparametric, moon2014multivariate}.

This paper addresses the ultimate learning to bound problem, which is to learn the tightest possible bound: the {\em exact} the Bayes error rate.  We call this the {\em learning to benchmark} problem.  Specifically, the contributions of this paper are as follows:
\begin{itemize}
\item A simple base learner of the Bayes error is proposed for general binary classification, its MSE convergence rate is derived, and it is shown to converge to the exact Bayes error probability (see Theorem~\ref{binaryestimatortheorem}). Furthermore, expressions for the rate of convergence are specified and we prove a central limit theorem for the proposed estimator (Theorem~\ref{thm:clt}).

\item An ensemble estimation technique based on Chebyshev nodes is proposed. Using this method a weighted ensemble of benchmark base learners is proposed having optimal (parametric) MSE convergence rates (see Theorem~\ref{thm:ensemble_theorem}). As contrasted to the ensemble estimation technique discussed in \citep{moon2018ensemble}, our method provides closed form solutions for the optimal weights based on Chebyshev polynomials (Theorem~\ref{thm:chebyshev_optimization}).

\item An extension of the ensemble benchmark learner is obtained for estimating the multiclass Bayes classification error rate and its MSE convergence rate is shown to achieve the optimal rate (see Theorem~\ref{multi-crlassestimatortheorem}).

\end{itemize}

The rest of the paper is organized as follows. In Section \ref{sec:binary}, we introduce our proposed Bayes error rate estimators for the binary classification problem. In Section \ref{sec:ensemble} we use the ensemble estimation method to improve the convergence rate of the base estimator.
We then address the multi-class classification problem in Section \ref{sec:multi-class}. In Section \ref{sec:experiment}, we conduct numerical experiments to illustrate the performance of the estimators. Finally, we discuss the future work in Section \ref{sec:conclusion}.

\section{Benchmark learning for Binary Classification}\label{sec:binary}
Our proposed learning to benchmark framework is based on an exact $f$-divergence representation (not a bound) for the minimum achievable binary misclassification error probability. First, in section \ref{sec:ball_graph_estimator} we propose an accurate estimator of the density ratio ($\eps$-ball estimator), and then in section \ref{section:base_estimator}, based on the optimal estimation for the density ratio, we propose a base estimator of Bayes error rate. 

\subsection{Density Ratio Estimator }\label{sec:ball_graph_estimator}

Consider the independent and identically distributed (i.i.d) sample realizations  $\bX_1=\big\{X_{1,1},X_{1,2},$ $\ldots,X_{1,N_1}\big\}\in \mathbb{R}^{N_1\times d}$ from $f_1$ and $\bX_2=\big\{X_{2,1},X_{2,2},\ldots,X_{2,N_2}\big\}\in \mathbb{R}^{N_2\times d}$  from $f_2$. Let $\eta:=N_2/N_1$ be the ratio of two sample sizes.
The problem is to estimate the density ratio $U(x):=\frac{f_1(x)}{f_2(x)}$ at each of the points of the set $\bX_2$. In this paper similar to the method of \citep{noshad2017direct} we use the ratio of counts of nearest neighbor samples from different classes to estimate the density ratio at each point. However, instead of considering the $k$-nearest neighbor points, we use the $\epsilon$-neighborhood (in terms of euclidean distance) of the points. This allows us to remove the extra bias due to the discontinuity of the parameter $k$ when using an ensemble estimation technique. As shown in Figure. \ref{fig:eps_ball_estimator}, $\eps$-ball density ratio estimator for each point $Y_i$ in $\bY$ (shown by blue points) is constructed by the ratio of the counts of samples in $\bX$ and $\bY$ which fall within $\eps$-distance of $Y_i$.

\begin{figure}[!h]
	\centering
	\includegraphics[width=0.7\textwidth]{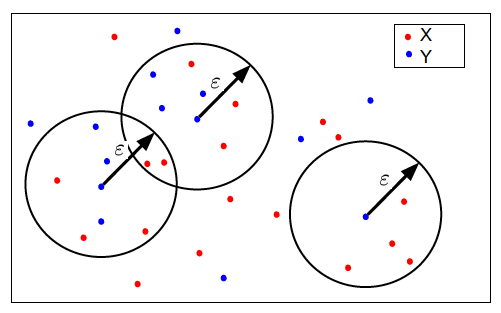}
	\caption{ $\eps$-ball density ratio estimator for each point $Y_i$ in $\bY$ (shown by blue points) is constructed by the ratio of the counts of samples in $\bX$ and $\bY$ which fall within $\eps$-distance of $Y_i$.
	\label{fig:eps_ball_estimator}}
\end{figure}

\begin{definition}
For each point $X_{2,i}\in\bX_2$, let $N_{1,i}^{(\eps)}$ (\emph{resp.} $N_{2,i}^{(\eps)}$) be the number of points belonging to $\bX_1$ (\emph{resp.} $\bX_2$) within the $\eps$-neighborhood ($\eps$-ball) of $X_{2,i}$. Then the density ratio estimate is given by
\begin{align}\label{eq:ratio_estimator}
    \widehat{U}^{(\eps)}(X_{2,i}):= \eta N_{1,i}^{(\eps)}\left/N_{2,i}^{(\eps)}\right..
\end{align}
\end{definition}

Sometimes in this paper we abbreviate $\widehat{U}(X_{2,i})$ as $\widehat{U}_i^{(\eps)}$. 

\subsection{Base learner of Bayes error}\label{section:base_estimator}
The Bayes error rate corresponding to class densities $f_1,f_2$, and the class probabilities vector $\bp=(p_1,p_2)$ is
\begin{align}
&\epsBayes_{\bp}(f_1,f_2)=\Pr(\CBayes(X)\ne T)\nonumber\\
=&\,\int_{p_1f_1(x)\le p_2f_2(x)} p_1f_1(x) dx+\hspace{-0.1cm}\int_{p_1f_1(x)\ge p_2f_2(x)} p_2f_2(x) dx,\label{eq:bayeserror8}
\end{align}
where $\CBayes(X)$ is the classifier mapping $\CBayes: \mathcal X \rightarrow \{1,2\}$. The Bayes error \eqref{eq:bayeserror8} can be expressed as
\begin{align}
\epsBayes_{\bp}(f_1,f_2)&\,=\frac{1}{2}\int p_1f_1(x)+p_2f_2(x) -|p_1f_1(x)-p_2f_2(x)|dx\nonumber\\
&\,=p_2+\frac{1}{2}\int (p_1f_1(x)-p_2f_2(x))-|p_1f_1(x)-p_2f_2(x)|dx\nonumber\\
&\,=\min(p_1,p_2)-\int f_2(x) t\bigg(\frac{f_1(x)}{f_2(x)}\bigg)dx\nonumber\\
&\,=\min(p_1,p_2)-\mathbb{E}_{f_2}\left[t\bigg(\frac{f_1(X)}{f_2(X)}\bigg)\right],\label{eq:BinaryBayesError}
\end{align}
where $$t(x):=\max(p_2-p_1x,0)-\max(p_2-p_1,0)$$
is a convex function. The expectation $\mathbb{E}_{f_2}\left[t\bigg(\frac{f_1(X)}{f_2(X)}\bigg)\right]$ is an $f$-divergence between density functions $f_1$ and $f_2$.
The $f$-divergence or Ali-Silvey distance, introduced in \citep{ali1966general}, is a measure of the dissimilarity between a pair of distributions.
Several estimators of $f$-divergences have been introduced \citep{berisha2016empirically,wang2005divergence,noshad2018hash,poczos2012nonparametric}.
Expressions for the bias and variance of these estimators are derived under assumptions that the function $t$ is differentiable, which is not true here. In what follows we will only need to assume that the divergence function $t$ is Lipschitz continuous.

We make the following assumption on the densities. Note that these are similar to the assumptions made in the previous work \citep{singh2014exponential,noshad2018rate,moon2018ensemble}.

{\bf Assumptions:}

{\bf A.1. }  The densities functions $f_1$ and $f_2$ are both lower bounded by $C_L$ and upper bounded by $C_U$ with $C_U\ge C_L>0$;

{\bf A.2. } The densities $f_1$ and $f_2$ are H\"older continuous with parameter $0< \gamma\le 1$, that is there exists constants $H_1,H_2>0$ such that
\begin{equation}
|f_i(x_1)-f_i(x_2)|\le H_i||x_1-x_2||^\gamma,
\end{equation}
for $i=1,2$ and $x_1,x_2\in \mathbb{R}$.

Explicit upper and lower bounds $C_U$ and $C_L$ must be specified for the implementation of the base estimator below. However, the lower and upper bounds do not need to be tight and only affect the convergence rate of the estimator. We conjecture that this assumption can be relaxed, but this is left for future work.

Define the base estimator of the Bayes error
 \begin{equation}\label{BayesEstimator}
\Best:=\min(\hat{p}_1,\hat{p}_2)-\frac{1}{N_2}\sum_{i=1}^{N_2}\tilde{t}\left(\widehat{U_i}\right),
\end{equation}
where $\tilde{t}(x):=\max(t(x),t(C_L/C_U))$, and empirical estimates vector $\hat{\mathbf{p}}=(\hat{p}_1,\hat{p}_2)$ is obtained from the relative frequencies of the class labels in the training set. $\widehat{U_i}$ is the estimation of the density ratio at point $\bX_{2,i}$, which can be computed based on $\eps$-ball estimates.

\begin{remark}
The definition of Bayes error in \eqref{eq:bayeserror8} is symmetric, however, the definition of Bayes error estimator in \eqref{BayesEstimator} is asymmetric with respect to $\bX_1$ and $\bX_2$. Therefore, we might get different estimations from $\Best$ and $\widehat{\mathcal{E}}_\epsilon(\bX_2,\bX_1)$, while both of these estimations asymptotically converge to the true Bayes error. It is obvious that any convex combination of $\Best$ and $\widehat{\mathcal{E}}_\epsilon(\bX_2,\bX_1)$ defined is also an estimator of the Bayes error (with the same convergence rate). In particular, we define the following symmetrized Bayes error estimator:
\begin{align}\label{eq:Bayes_est_symmetric}
    \mathcal{E}^*_\epsilon(\bX_2,\bX_1) &:= \frac{N_2}{N} \Best + \frac{N_1}{N} \widehat{\mathcal{E}}_\epsilon(\bX_2,\bX_1)\nonumber\\
    &= \min(\hat{p}_1,\hat{p}_2)-\frac{1}{N}\sum_{i=1}^{N}\tilde{t}\left(\widehat{U_i}\right),
\end{align}
where consistent with the definition in \eqref{eq:ratio_estimator}, for the points in $\bX_1$, $\widehat{U}^{(\eps)}_i$ is defined as the ratio of the $\eps$-neighbor points in $\bX_2$ to the number of points in $\bX_1$, while for the points in $\bX_2$ is defined as the ratio of the points in $\bX_1$ to the number of points in $\bX_2$:
\begin{align}
    \widehat{U}^{(\eps)}_i:= 
    \begin{cases}
       \eta N_{1,i}^{(\eps)}\left/N_{2,i}^{(\eps)}\right. &\quad 1\le i\le N_2\\
        N_{2,i}^{(\eps)}\left/\eta N_{1,i}^{(\eps)}\right. &\quad N_2\le i\le N.\\
     \end{cases}
\end{align}
\end{remark}

\begin{algorithm} \label{algo:Base_learner}
\DontPrintSemicolon
\SetKwInOut{Input}{Input}\SetKwInOut{Output}{Output}
\Input{Data sets $\bX=\{X_1,...,X_{N_1}\}$, $\bY=\{Y_1,...,Y_{N_2}\}$}

\BlankLine
 
$\bZ\leftarrow \bX \cup \bY$    \;
\For {each point $Y_i$ in $Y$}{
		$\bS_i$: Set of $\eps$-ball points of $Y_i$ in $\bZ$\\
		 $\widehat{U}_i \leftarrow |\bS_i\cap \bX|/|\bS_i \cap \bY|$}
$\mathcal{E}^*_\epsilon(\bX_2,\bX_1) \leftarrow\min(N_1,N_2)/(N_1+N_2)-\frac{1}{N}\sum_{i=1}^{N}\tilde{t}\left(\widehat{U_i}\right),$

\Output{$\mathcal{E}^*_\epsilon(\bX_2,\bX_1) $}

\caption{Base Learner of Bayes Error}
\end{algorithm}


\begin{remark}
The $\eps$-ball density ratio estimator is equivalent to the ratio of plug-in kernel density estimators with a top-hat filter and bandwidth $\eps$.
\end{remark}

\subsection{Convergence Analysis}

The following theorem states that this estimator asymptotically converges in $L^2$ norm to the exact Bayes error as $N_1$ and $N_2$ go to infinity in a manner $N_2/N_1\to \eta$,  with an MSE rate of $O(N^{-\frac{2\gamma}{\gamma+d}})$.
\begin{theorem}\label{binaryestimatortheorem}
Under the Assumptions on $f_1$ and $f_2$ stated above, as $N_1,N_2\to \infty$ with $N_2/N_1\to \eta$,
\begin{equation}
\Best\overset{\scriptscriptstyle{L^2}}{\to} \epsBayes_{\bp}(f_1,f_2),
\end{equation}
where $\overset{\scriptscriptstyle{L^2}}{\to}$ denotes ``convergence in $L^2$ norm''. Further, the bias of $\mathcal{E}(\bX_1,\bX_2)$ is
\begin{align} \label{bias_estimator}
\mathbb{B}\left[\Best\right]= O\left(\epsilon^\gamma\right)+ O\left(\epsilon^{-d}N_1^{-1}\right),
\end{align}
where $\eps$ is the radius of the neighborhood ball.

In addition, the variance of $\Best$ is
\begin{align} \label{variance}
\mathbb{V}\left[\Best\right]= O\left(1/\min(N_1,N_2)\right).
\end{align}
\end{theorem}


\begin{proof}
Since according to \eqref{eq:BinaryBayesError} the Bayes error rate $\epsBayes$ can be written as
an $f$-divergence, it suffice to derive the bias and variance of the $\eps$-ball estimator of the divergence. The details are given in Appendix. \ref{sec:app_binary_convergence}.
\end{proof}

In the following we give a theorem that establishes the Gaussian convergence of the estimator proposed in equation \eqref{BayesEstimator}.

\begin{theorem}\label{thm:clt}
    Let $\eps\to 0$ and $\frac{1}{\eps^dN}\to 0$. If $S$ be a standard normal random variable with mean $0$ and variance $1$, then,   
    \begin{align}
        Pr\left(\frac{\Best -  \mathbb{E}\left[\Best\right]}{\sqrt{\mathbb{V}\left[\Best\right]}}\leq t \right) \to Pr(S\leq t)
    \end{align}
\end{theorem}
\textbf{Proof:} The proof is based on the Slutsky's Theorem and Efron-Stein inequality  and is discussed in details in Appendix. \ref{sec:proof_CLT}.

\section{Ensemble of Base Learners}\label{sec:ensemble}

It has long been known that ensemble averaging of base learners can improve the accuracy and stability of learning algorithms \citep{dietterich2000ensemble}.
In this work in order to achieve the optimal parametric MSE rate of $O(1/N)$, we propose to use an ensemble estimation technique. The ensemble estimation technique has previously used in estimation of $f$-divergence and mutual information measures \citep{moon2018ensemble,moon2016improving, noshad2018hash}. However, the method used by these articles depends on the assumption that the function $f$ of the divergence (or general mutual information) measure is differentiable everywhere within its the domain. As contrasted to this assumption, function $t(x)$ defined in equation \eqref{eq:BinaryBayesError} is not differentiable at $x=p_1/p_2$, and as a result, using the ensemble estimation technique considered in the previous work is difficult. A simpler construction of the ensemble Bayes error estimation is discussed in section \ref{sec:ensemble_construction}. Next, in section \ref{sec:Chebyshev} we propose an optimal weight assigning method based on Chebyshev polynomials.

\subsection{Construction of the Ensemble Estimator}\label{sec:ensemble_construction}
 Our proposed ensemble benchmark learner constructs a weighted average of $L$ density ratio estimates defined in \eqref{eq:ratio_estimator}, where each density ratio estimator uses a different value of $\epsilon$.

\begin{definition}
Let $\widehat{U}_i^{(\eps_j)}$ for $j\in \{1,...,L\}$ be $L$ density ratio estimates with different parameters $(\eps_j)$ at point $Y_i$. For a fixed weight vector $\bw:=(w_1,w_2,\ldots,w_L)^T$, the ensemble estimator is defined as
\begin{align}
&\mathcal{F}(\bX_1,\bX_2)=\min(\hatp_1,\hatp_2)-\frac{1}{N_2}\sum_{i=1}^{N_2}\left[\max(\hatp_2-\hatp_1\widehat{U}_i^{\bw} ,0)-\max(\hatp_2-\hatp_1,0)\right],\label{eq:ensemble_estimator_binary}
\end{align}
where for the weighted density ratio estimator, $\widehat{U}_i^{\bw}$ is defined as
\begin{align}\label{def_Ui}
    \widehat{U}_i^{\bw}:=\sum_{l=1}^{L} w_{l}\widehat{U}_i^{(\eps_l)}.
\end{align}

\end{definition}



\begin{remark}
The construction of this ensemble estimator is fundamentally different from standard ensembles of base estimators  proposed before and, in particular, different from the methods proposed in \citep{moon2018ensemble,noshad2018hash}. These standard methods average the base learners whereas the ensemble estimator \eqref{eq:ensemble_estimator_binary} averages over the argument (estimated likelihood ratio $f_1/f_2$) of the base learners.
\end{remark}

Under additional conditions on the density functions, we can find the weights $w_l$ such that the ensemble estimator in \eqref{eq:ensemble_estimator_binary} achieves the optimal parametric MSE rate $O(1/N)$. Specifically, assume that 1) the density functions $f_1$ and $f_2$ are both H\"older continuous with parameter $\gamma$ and continuously differentiable of order $q=\lfloor \gamma \rfloor\ge d$ ,and 2) the $q$-th derivatives are H\"older continuous with exponent $\gamma':=\gamma-q$. These are similar to assumptions that have been made in the previous work \citep{moon2018ensemble,singh2014exponential,noshad2018hash}. We prove that if the weight vector $\bw$ is chosen according to an optimization problem, the ensemble estimator can achieve the optimal parametric MSE rate $O(1/N)$.

\begin{theorem}\label{thm:ensemble_theorem}
Let $N_1,N_2\to \infty$ with $N_2/N_1\to \eta$. Also let  $\widehat{U}_i^{(\eps_j)}$ for $j\in \{1,...,L\}$ be $L$ ($L>d$) density ratio estimates with bandwidths $\eps_j:=\xi_jN_1^{-1/2d}$ at the points $Y_i$.
Define the weight vector  $\bw=(w_1,w_2,\ldots,w_L)^T$  as the solution to the following optimization problem:
\begin{align}
\min_{\bw} &\qquad ||\bw||_2\label{eq:optimization}\\
\textit{subject to} &\qquad \sum_{l=1}^L w_l=1 \ \textrm{  and  }\ \sum_{l=1}^Lw_l\cdot \xi_l^i=0,\qquad \forall i=1,\ldots,d .\nonumber
\end{align}

Then, under the assumptions stated above the ensemble estimator defined in \eqref{eq:ensemble_estimator_binary} satisfies,
\begin{equation}
\mathcal{F}(\bX_1,\bX_2)\overset{\scriptscriptstyle{L^2}}{\to}\epsBayes_{\bp}(f_1,f_2),
\end{equation}
with the MSE rate $O(1/N_1)$.
\end{theorem}

\begin{proof}
See Appendix \ref{section:proofofensemble_theorem}.
\end{proof}

One simple choice for $\xi_l$ is an arithmetic sequence as $\xi_l:=l$. With this setting the optimization problem in  the following optimization problem:

\begin{align}\label{eq:optimization_arithmetic}
\min_{\bw} &\qquad ||\bw||_2\\
\textit{subject to} &\qquad \sum_{l=1}^L w_l=1 \ \textrm{  and  }\ \sum_{l=1}^Lw_l\cdot l^i=0,\qquad \forall i=1,\ldots,d .
\end{align}
Note that the optimization problem in \eqref{eq:optimization_arithmetic} does not depend on the data sample distribution and only depends on its dimension. Thus, it can be solved offline. In larger dimensions, however, solving the optimization problem can be computationally difficult. In the following we provide an optimal weight assigning approach based on Chebyshev polynomials that reduces computational complexity and leads to improved stability. We use the orthogonality properties of the Chebyshev polynomials to derive closed form solutions for the optimal weights in \eqref{eq:optimization}.

\subsection{Chebyshev Polynomial Approximation Method for Ensemble Estimation}\label{sec:Chebyshev}

Chebyshev polynomials are frequently used in function approximation theory \citep{kennedy2004approximation}. 
We denote the Chebyshev polynomials of the first kind defind in interval $[-1,1]$ by $T_n$, where $n$ is the degree of the polynomial. An important feature of Chebyshev polynomials is that the roots of these polynomials are used as polynomial interpolation points. 
We define the shifted Chebyshev polynomials with a parameter $\alpha$ as $T^\alpha_n(x):[0,\alpha]\to \mathbb{R}$ in terms of the standard Chebyshev polynomials as

\begin{align}
    T^\alpha_n(x) = T_n(\frac{2x}{\alpha}-1).
\end{align}
We denote the roots of $T^\alpha_n(x)$ by $s_i, i\in \{1,...,n\}$. In this section we formulate the ensemble estimation optimization in equation \eqref{eq:optimization} in the Chebyshev polynomials basis and we propose a simple closed form solution to this optimization problem. 
This is possible by setting the parameters of the base density estimators $\eps_l$ proportional to the Chebyshev nodes $s_l$. Precisely, in equation \eqref{eq:optimization} we set
\begin{align}
    \xi_l:=s_l.
\end{align}

\begin{theorem}\label{thm:chebyshev_optimization}

For $L>d$, the solutions of the optimization problem in \eqref{eq:optimization} for $\xi_l:=s_l$ are given as:

\begin{align} \label{eq:Chb_opt_thm}
    w_i = \frac{2}{L}\sum_{k=0}^{d}T^\alpha_k(0) T^\alpha_k(s_i) - \frac{1}{L}  \qquad \forall i \in \{0,...,L-1\}.
\end{align}
where $s_i, i\in \{0,...,L-1\}$ are roots of $T^\alpha_L(x)$ given by

\begin{align}\label{eq:roots}
    s_{k}=\frac{\alpha}{2}\cos \left(\left(k+\frac{1}{2}\right) \frac{\pi}{L}\right)+\frac{\alpha}{2}, \quad k=0, \ldots, L-1
\end{align}

\end{theorem}

\begin{proof}
The proof of Theorems \ref{thm:chebyshev_optimization} can be found in Appendix~\ref{sec:chebyshevproof}.
\end{proof}

\section{Benchmark Learning for Multi-class Classification}\label{sec:multi-class}
Consider a multi-class classification problem with $\lambda$ classes having respective density functions
$f_1,f_2,\ldots,f_\lambda$.
The Bayes error rate for the multi-class classification is
\begin{align}
  &\epsBayes_{\bp}(f_1,f_2,\ldots,f_\lambda )\nonumber\\
  &=1-\int \left[\max_{1\le i\le \lambda } p_i f_i(x)\right] dx\nonumber\\
  &=1-p_1-\sum_{k=2}^\lambda \int\left[\max_{1\le i\le k} p_i f_i(x)-\max_{1\le i\le k-1} p_i f_i(x)\right] dx\nonumber\\
  &=1-p_1-\sum_{k=2}^\lambda \int\max\left(0,p_k-\max_{1\le i\le k-1} p_i f_i(x)/f_k(x)\right)f_k(x) dx\nonumber\\
  &=1-p_1-\sum_{k=2}^\lambda  \int t_k\left(\frac{f_1(x)}{f_k(x)},\frac{f_2(x)}{f_k(x)},\ldots,\frac{f_{k-1}(x)}{f_k(x)}\right)f_k(x)dx,\label{eq:Multi-classBayesError}
\end{align}
where
$$t_k(x_1,x_2,\ldots,x_{k-1}):=\max\left(0,p_k-\max_{1\le i\le k-1}p_ix_i\right).$$

We denote the density fractions $\frac{f_i(x)}{f_j(x)}$ in the above equation by $U_{(i/j)}(x)$. Let $\widehat{U}^{\bw}_{(i/j)}(x)$ denote the ensemble estimates of $U_{(i/j)}(x)$ using the $\eps$-ball method, similar to the estimator defined in \eqref{def_Ui}.
Thus, we propose the following direct estimator of $\epsBayes_{\bp}(f_1,f_2,\ldots,f_\lambda )$ as follows:
\begin{align}
&\,\mathcal{H}(\bX_1,\bX_2,\ldots,\bX_\lambda ):=1-p_1-\label{Multi-classBayesEstimator}\\
&\qquad \sum_{l=2}^\lambda  \frac{1}{N_l}\sum_{i=1}^{N_l} \tilde{t}\bigg(\widehat{U}^{\bw}_{(1/l)}(X_{l,i}),\widehat{U}^{\bw}_{(2/l)}(X_{l,i}),\ldots,\widehat{U}^{\bw}_{(l-1/l)}(X_{l,i})\bigg)\nonumber,
\end{align}
where
\begin{align*}
&\tilde{t}_k(x_1,x_2,\ldots,x_{k-1}):=\max\left\{t_k(x_1,x_2,\ldots,x_{k-1}),t_k(C_L/C_U,\ldots,C_L/C_U)\right\}.
\end{align*}

Since $t$ is elementwise Lipschitz continuous, we can easily generalize the argument used in the proof of Theorem~\ref{binaryestimatortheorem} to obtain the convergence rates for the multiclass case. Similar to the assumptions of the ensemble estimator for the binary case in section \ref{sec:ensemble_construction}, we assume that 1) the density functions $f_1,f_2,...,f_{\lambda}$ are both H\"older continuous with parameter $\gamma$ and continuously differentiable of order $q=\lfloor \gamma \rfloor\ge d$ and 2) the $q$-th derivatives are H\"older continuous with exponent $\gamma':=\gamma-q$.


\begin{theorem}\label{multi-crlassestimatortheorem}
As $N_1,N_2,\ldots,N_\lambda  \to \infty$ with $N_l/N_j \to \eta_{j,l}$ for $1\le j<l\le \lambda $ and $N^*=\max(N_1,N_2,\ldots,N_\lambda )$,
\begin{align}
\mathcal{H}_k(\bX_1,\bX_2, \ldots,\bX_\lambda )\overset{\scriptscriptstyle{L^2}}{\to} \epsBayes_{\bp}(f_1,f_2,\ldots,f_\lambda ).
\end{align}
The bias and variance of $\mathcal{H}_k(\bX_1,\bX_2,\ldots,\bX_\lambda )$ are
\begin{align}
\mathbb{B}\left[\mathcal{H}_k(\bX_1,\bX_2,\ldots,\bX_\lambda )\right]&= O\left({\lambda}/{\sqrt{N^*}}\right),\label{bias_estimator_multi}\\
\mathbb{V}\left[\mathcal{H}_k(\bX_1,\bX_2,\ldots,\bX_\lambda )\right]&= O\left({\lambda ^2}/{N^*}\right).\label{variance_multi_class}
\end{align}

\end{theorem}

\begin{proof}
See Appendix \ref{section:proofofmulti-classestimatortheorem}.
\end{proof}

\begin{remark}
Note that the estimator $\mathcal{H}_k$ (\ref{Multi-classBayesEstimator}) depends on the ordering of the classes, which is arbitrary. However the asymptotic MSE rates do not depend on the particular class ordering.
\end{remark}
\begin{remark}
In fact, \eqref{eq:Multi-classBayesError} can be transformed into
\begin{align}
  &\epsBayes_{\bp}(f_1,f_2,\ldots,f_\lambda )=1-p_1-\sum_{k=2}^\lambda p_k\int\max\left(0,1-h_k(x)/f_k(x)\right)f_k(x) dx,
\end{align}
where $h_k(x):=\max_{1\le i\le k-1} p_i f_i(x)/p_k$. That shows that the Bayes error rate is actually a linear combination of $(\lambda -1)$ $f$-divergences.
\end{remark}

\begin{remark}
The function $t_k$ is not a properly defined generalized $f$-divergence \citep{duchi2016multiclass}, since
$t_k\left(\frac{p_k}{p_1},\frac{p_k}{p_2},\ldots,\frac{p_k}{p_{k-1}}\right)=0,$
while $t_k(1,1,\ldots,1)$ is not necessarily equal to $0$.
\end{remark}


\section{Numerical Results}\label{sec:experiment}
We apply the proposed benchmark learner on several numerical experiments for binary and multi-class classification problems. We perform experiments on different simulated datasets with dimensions of up to $d=100$.
We compare the benchmark learner to previous lower and upper bounds on the Bayes error based on HP-divergence \eqref{HP_bound}, as well as to a few powerful classifiers on different classification problem. The proposed benchmark learner is applied on the MNIST dataset with $70k$ samples and $784$ features, learning theoretically the best achievable classification error rate. This is compared to reported performances of state of the art deep learning models applied on this dataset. Extensive experiments regarding the sensitivity with respect to the estimator parameter, the difference between the arithmetic and Chebyshev optimal weights and comparison of the corresponding ensemble benchmark learner performances, and comparison to the previous bounds on the Bayes error and classifiers on various simulated datasets with Gaussian, beta, Rayleigh and concentric distributions are provided in Appendix~\ref{sec:supp_num_results}. 


Figure \ref{fig:cmp_bound} compares the optimal benchmark learner with the Bayes error lower and upper bounds using HP-divergence, for
a binary classification problems with $10$-dimensional isotropic normal distributions with identity covariance matrix, where the means are shifted by 5 units in the first dimension. 
While the HP-divergence bounds have a large bias, the proposed benchmark learner converges to the true value by increasing sample size.


\begin{figure}
  \centering
  \includegraphics[width=0.7\textwidth]{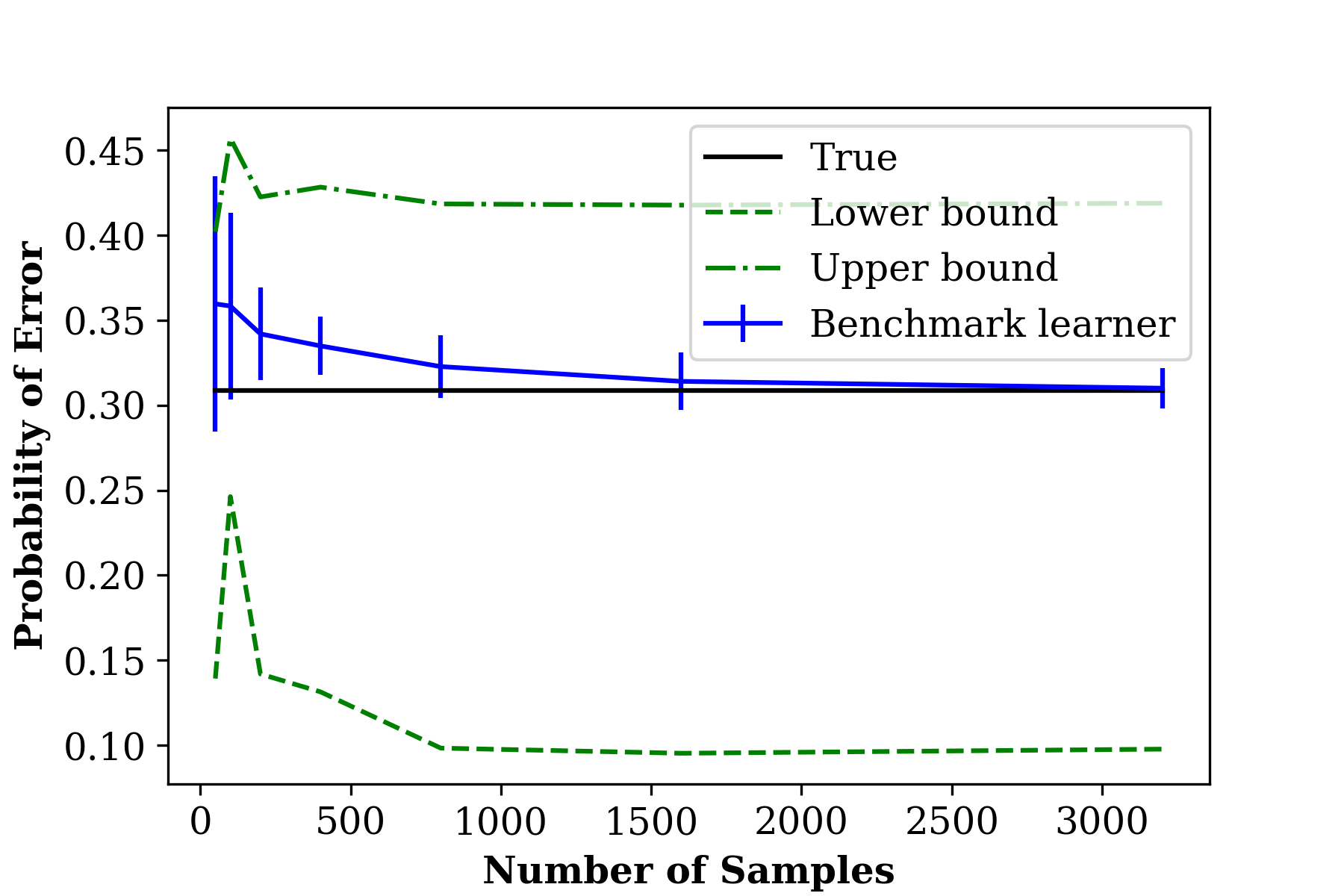}
\caption{Comparison of the optimal benchmark learner (Chebyshev method) with the Bayes error lower and upper bounds using HP-divergence, for 
 a binary classification problems with $10$-dimensional isotropic normal distributions with identity covariance matrix, where the means are shifted by 5 units in the first dimension. While the HP-divergence bounds have a large bias, the proposed benchmark learner converges to the true value by increasing sample size.}
\label{fig:cmp_bound}
\end{figure}

In Figure~\ref{fig:cmp_clsfr_concentric} we compare the optimal benchmark learner (Chebyshev method) with XGBoost, Random Forest and deep neural network (DNN) classifiers, for a $4$-class classification problem $20$-dimensional concentric distributions. Note that as shown in (b) the concentric distributions are resulted by dividing a Gaussian distribution with identity covariance matrix into four quantiles such that each class has the same number of samples. The DNN classifier consists of 5 hidden layers with $[20,64,64,10,4]$ neurons and ReLU activations. Also in each layer a dropout with rate $0.1$ is applied to diminish the overfitting. The network is trained using Adam optimizer and is trained for $150$ epochs.

\begin{figure}[H]
  \centering
  \subfigure[Four classes with concentric distributions]{\centering\qquad\quad\includegraphics[width=0.65\textwidth]{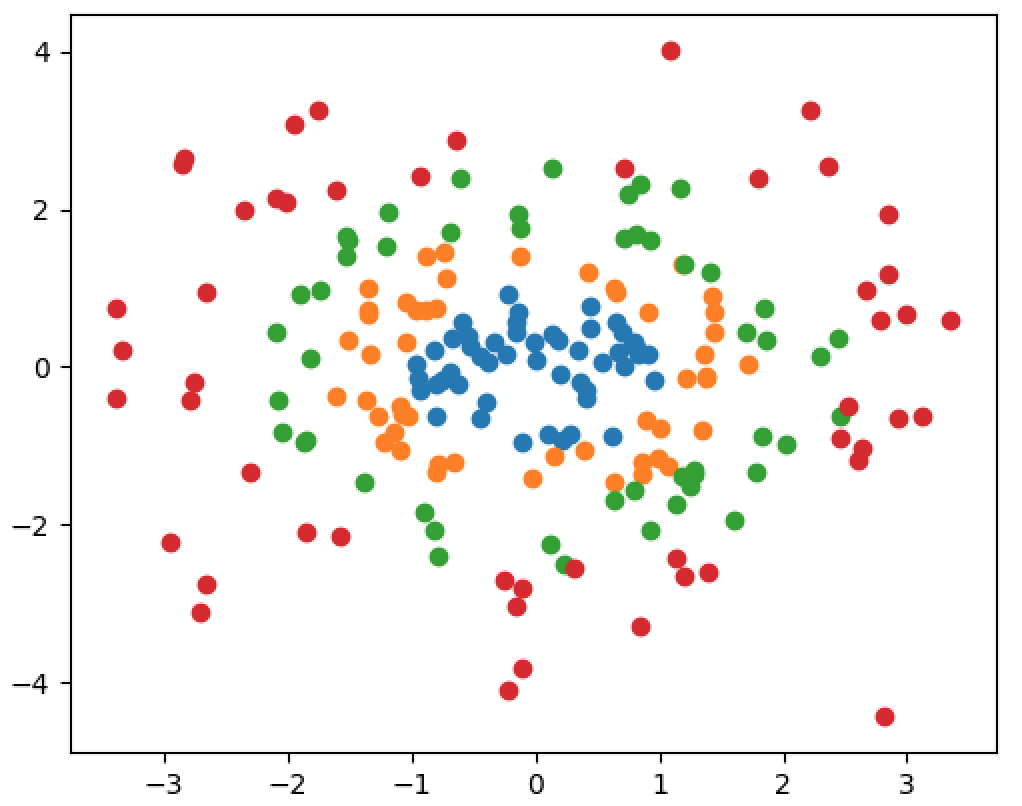}}\newline
  \subfigure[Benchmark learner compared to a 5-layer DNN, XGBoost and Random Forest classifiers for the concentric distributions]{\centering\includegraphics[width=0.7\textwidth]{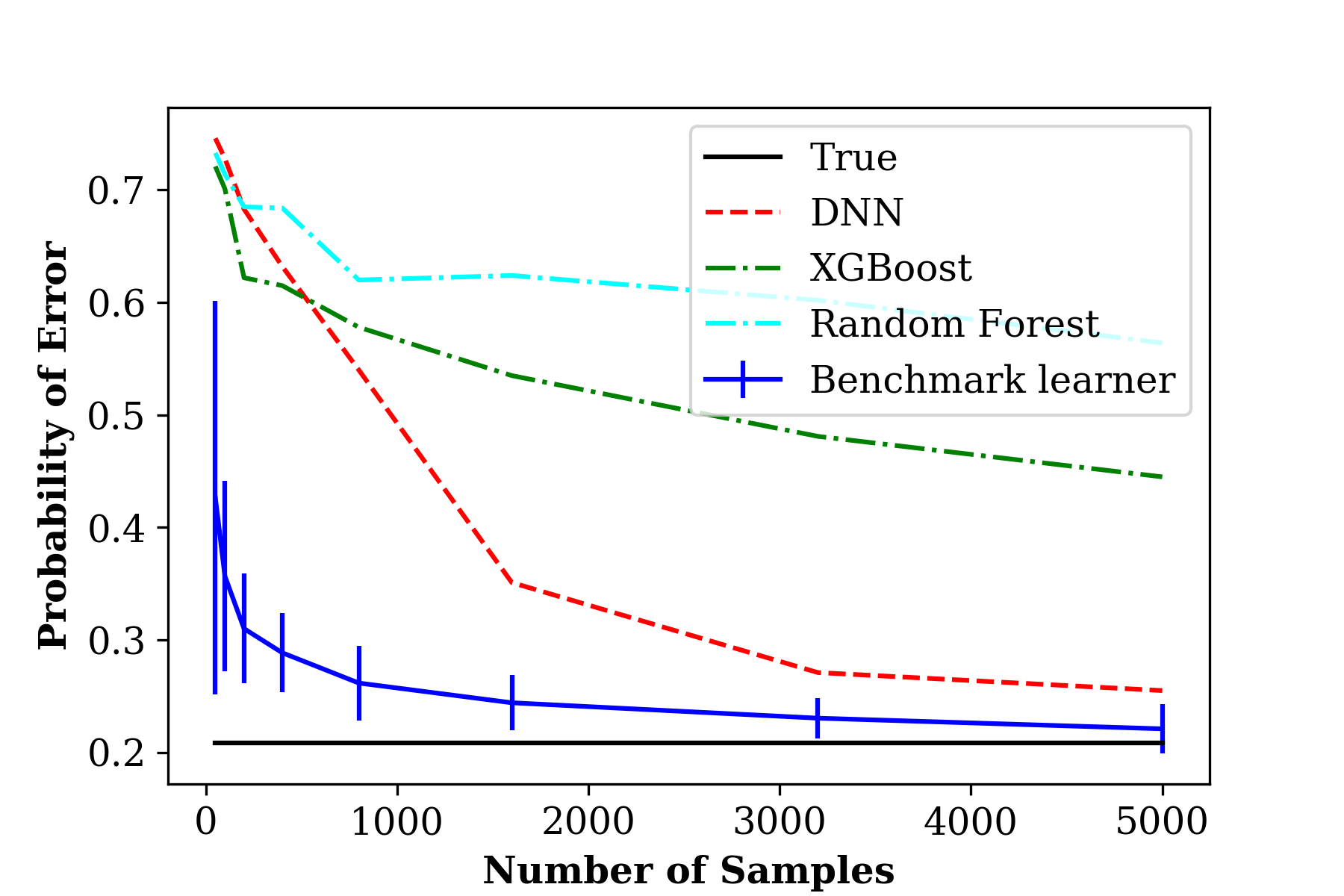}}
\caption{Comparison of the optimal benchmark learner (Chebyshev method) with a 5-layer DNN, XGBoost and Random Forest classifiers, for a $4$-class classification problem $20$-dimensional concentric distributions. Note that as shown in (b), the concentric distributions are resulted by dividing a Gaussian distribution with identity covariance matrix into four quantiles such that each class has the same number of samples. The benchmark learner predicts the Bayes error rate better than the DNN, XGBoost and Random Forest classifiers.}
\label{fig:cmp_clsfr_concentric}
\end{figure}


Further, we compute the benchmark learner for the MNIST dataset with 784 dimensions and 60,000 samples. In Table \ref{tab:MNIST_results} we compare the estimated benchmark learner with the reported state of the art convolutional neural network classifiers with 60,000 training samples. Note that according to the online report \citep{MNIST_reported} the listed models achieve the best reported classification performances.


\begin{table}\label{tab:MNIST}
 \begin{tabular}{|l|l|l|}
 \hline
 Papers & Method  & Error rate \\ [0.5ex]
 \hline\hline
 \citep{cirecsan2010deep} & Single 6-layer DNN  &0.35\%  \\
 \hline
\citep{ciresan2011flexible} & Ensemble of 7 CNNs and training data expansion &0.27\% \\
 \hline
 \citep{cirecsan2012multi} & Ensemble of 35 CNNs &0.23\%  \\
 \hline
 \citep{wan2013regularization} & Ensemble of 5 CNNs and DropConnect regularization &0.21\%  \\
 \hline
 Benchmark learner & Ensemble $\epsilon$-ball estimator &0.14\%  \\ [1ex]
 \hline
\end{tabular}
\caption{Comparison of error probabilities of several the state of the art deep models with the benchmark learner, for the MNIST handwriting image classification dataset}
\label{tab:MNIST_results}
\end{table}


The benchmark learner can also be used as a stopping rule for deep learning models. This is 
demonstrated in figures \ref{fig:DNN_stop_vs_N} and \ref{fig:DNN_stop_vs_iter}. In both of these figures we consider a $3$-class classification problem with $30$-dimensional Rayleigh distributions with parameters $a=0.7,1.0,1.3$. We train a DNN model consisting of 5 layers with $[30,100,64,10,3]$ neurons and RELU activations. Also in each layer a dropout with rate $0.1$ is applied to diminish the overfitting. In Figure. \ref{fig:DNN_stop_vs_N} we feed in different numbers of samples and compare the error rate of the classifier with the proposed benchmark learner. The network is trained using Adam optimizer for 150 epochs. At around $500$ samples, the error rate of the trained DNN is within the confidence interval of the benchmark learner, and one can probably stop increasing the sample number since the error rate of the DNN is close enough to the Bayes error rate. In Figure. \ref{fig:DNN_stop_vs_iter} we feed in $2000$ samples to the network and plot the error rate for different training epochs. At around $80$ epochs, the error rate of the trained DNN is within the confidence interval of the benchmark learner, and we can stop training the network since the error rate of the DNN is close enough to the Bayes error rate. 

\begin{figure}[!h]
	\centering
	\includegraphics[width=0.7\textwidth]{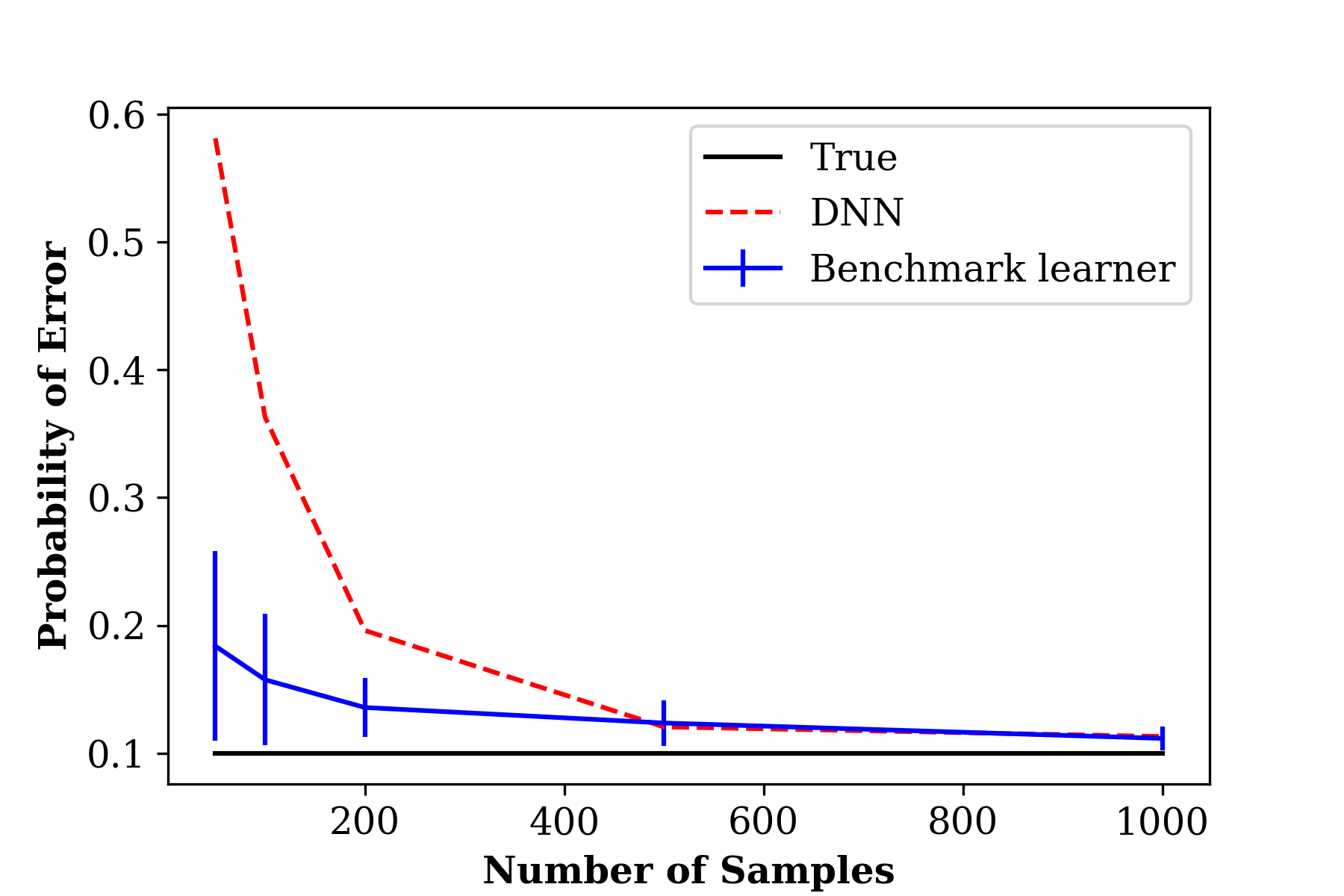}
	\caption{Error rate of a DNN classifier compared to the benchmark learner for a $3$-class classification problem with $30$-dimensional Rayleigh distributions with parameters $a=0.7,1.0,1.3$. We train a DNN model consisting of 5 layers with $[30,100,64,10,3]$ neurons and RELU activations. Also in each layer a dropout with rate $0.1$ is applied to diminish the overfitting. We feed in different numbers of samples and compare the error rate of the classifier with the proposed benchmark learner. The network is trained for about $50$ epochs. At around $500$ samples, the error rate of the trained DNN is within the confidence interval of the benchmark learner, and one can probably stop increasing the sample number since the error rate of the DNN is close enough to the Bayes error rate. \label{fig:DNN_stop_vs_N}}
\end{figure}

\begin{figure}[!h]
	\centering
	\includegraphics[width=0.7\textwidth]{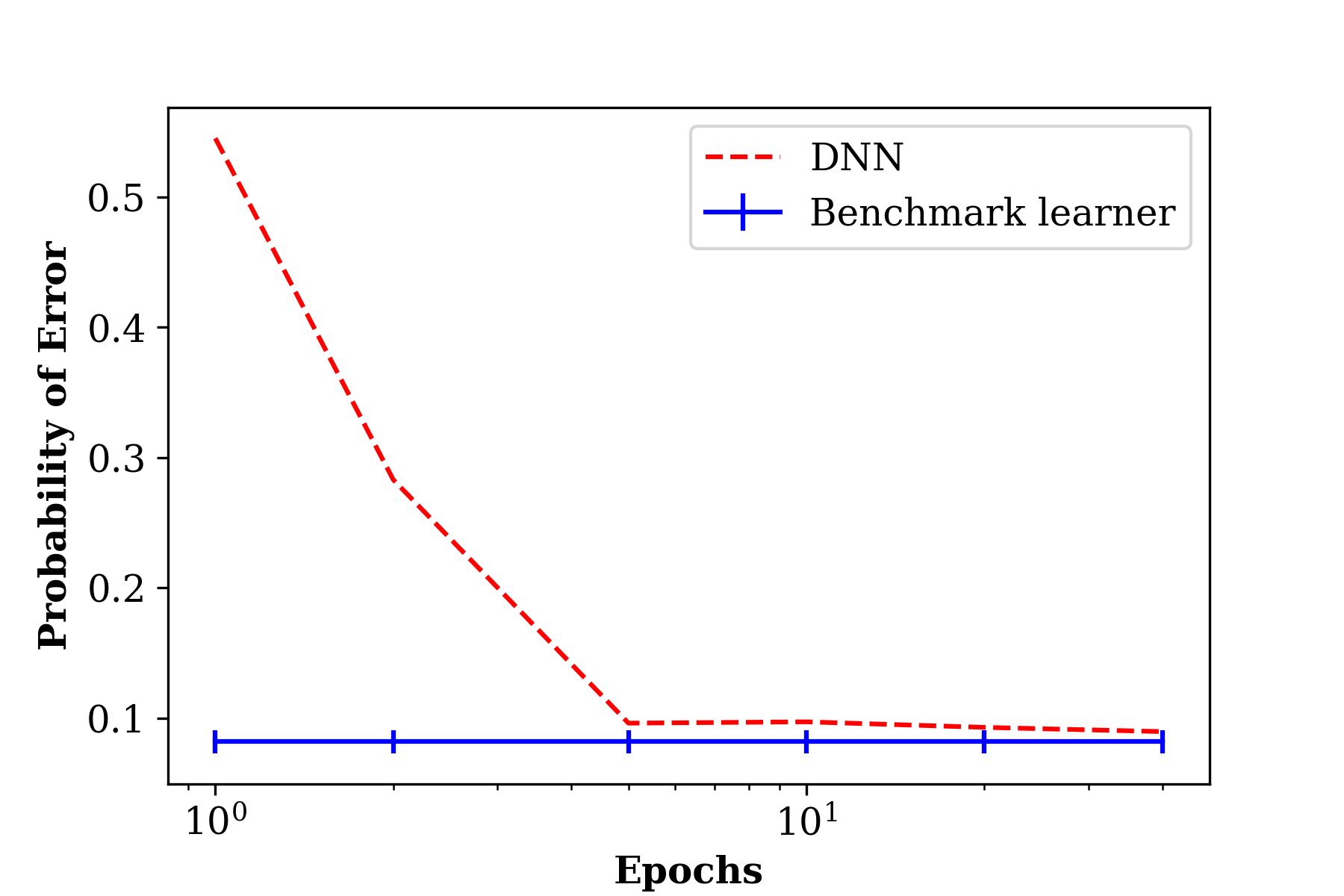}
	\caption{Error rate of a DNN classifier compared to the benchmark learner for a $3$-class classification problem with $30$-dimensional Rayleigh distributions with parameters $a=0.7,1.0,1.3$. We train a DNN model consisting of 5 layers with $[30,100,64,10,3]$ neurons and RELU activations. Also in each layer a dropout with rate $0.1$ is applied to diminish the overfitting. We feed in $2000$ samples to the network and plot the error rate for different training epochs. At around $40$ epochs, the error rate of the trained DNN is within the confidence interval of the benchmark learner, and we can stop training the network since the error rate of the DNN is close enough to the Bayes error rate. 
	\label{fig:DNN_stop_vs_iter}}
\end{figure}


\section{Conclusion}\label{sec:conclusion}
In this paper,  a new framework, benchmark learning, was proposed that learns the Bayes error rate for classification problems. An ensemble of base learners was developed for binary classification and it was shown to converge to the exact Bayes error probability with optimal (parametric) MSE rate. An ensemble estimation technique based on Chebyshev polynomials was proposed that provides closed form expressions for the optimum weights of the ensemble estimator.
Finally, the framework was extended to multi-class classification and the proposed benchmark learner was shown to converge to the Bayes error probability with optimal MSE rates.


\newpage
\appendix


\section{Proof of Theorem \ref{binaryestimatortheorem}}\label{sec:app_binary_convergence}

Theorem \ref{binaryestimatortheorem} consists of two parts: bias and variance bounds. For the bias proof, from equation \eqref{BayesEstimator} we can write

\begin{align}\label{eq:bias_proof_1}
    \mathbb{E}\left[\Best\right]&= \mathbb{E}\left[\min(\hat{p}_1,\hat{p}_2)-\frac{1}{N_2}\sum_{i=1}^{N_2}\tilde{t}\left(\widehat{U_i}\right)\right]\nonumber\\
    & = \min(\hat{p}_1,\hat{p}_2)-\frac{1}{N_2}\sum_{i=1}^{N_2}\mathbb{E}\left[\tilde{t}\left(\widehat{U_i}\right)\right]\nonumber\\
    & = \min(\hat{p}_1,\hat{p}_2)-\mathbb{E}_{X_{2,1}\sim f_2}\mathbb{E}\left[\tilde{t}\left(\widehat{U}_1\right)\vert X_{2,1}\right]
\end{align}

Now according to equation (33) of \citep{noshad2018hash}, for any region for which its geometry is independent of the samples and the largest diameter within the region is equal to $c\eps$, where $c$ is a constant, then we have

\begin{align}\label{eq:hash_bias_result}
    \mathbb{E}\left[\tilde{t}\left(\widehat{U}_1\right)\vert X_{2,1}=x\right] = \tilde{t}\left(\frac{f_1(x)}{f_2(x)}\right) + O\left(\eps^\gamma\right) + O\left(\frac{1}{N\eps^d}\right). 
\end{align}

Thus, plugging \eqref{eq:hash_bias_result} in \eqref{eq:bias_proof_1} results in

\begin{align}\label{eq:bias_proof_2}
    \mathbb{E}\left[\Best\right]&=
    \min(\hat{p}_1,\hat{p}_2)-\mathbb{E}_{f_2}\left[\tilde{t}\bigg(\frac{f_1(X)}{f_2(X)}\bigg)\right] + O\left(\eps^\gamma\right) + O\left(\frac{1}{N\eps^d}\right),
\end{align}
which completes the bias proof.

\begin{remark}
It can easily be shown that if we use the NNR density ratio estimator (defined in \citep{noshad2017direct}) with parameter $k$, the Bayes error estimator defined in \eqref{BayesEstimator} achieves the bias rate of $O\left(\left(\frac{k}{N}\right)^{\gamma/d}\right) + O\left(\frac{1}{k}\right)$.
\end{remark}


The approach for the proof of the variance bound is similar to the Hash-based estimator \citep{noshad2018hash}. 
Consider the two sets of nodes $X_{1,i}$, $1\leq i \leq N_1$ and $X_{2,j}$, $1\leq j \leq N_2$. For simplicity we assume that $N_1 = N_2$, however, similar to the variance proofs in \citep{noshad2017direct,noshad2018hash}, by considering a number of virtual points one can easily extend the proof to general $N_1$ and $N_2$. Let $Z_i:=(X_{1,i},X_{2,i})$. For using the Efron-Stein inequality on $\bZ:=(Z_1,...,Z_{N_1})$, we consider another independent copy of $Z$ as  $\bZ':=(Z'_1,...,Z'_{N_1})$ and define $\bZ^{(i)}:=(Z_1,...,Z_{i-1},Z'_i,Z_{i+1},...,Z_{N_1})$. In the following we use the Efron-Stein inequality. Note that we use the shorthand $\mathcal{E}{(\bZ)}:=\Best$.

\begin{align}\label{variance_main_2}
\mathbb{V}\left[\mathcal{E}{(\bZ)}\right] &\leq \frac{1}{2} \sum_{i=1}^{N_1} \mathbb{E}\left[\left(\mathcal{E}{(\bZ)}-\mathcal{E}{(\bZ^{(i)})}\right)^2\right]\nonumber\\
&= \frac{N_1}{2} \mathbb{E}\left[\left(\mathcal{E}{(\bZ)}-\mathcal{E}{(\bZ^{(1)})}\right)^2\right]  \nonumber\\
&\leq \frac{N_1}{2}\mathbb{E}{\left(\frac{1}{N_1}\sum_{i=1}^{N_1} \widetilde{t}\left(\frac{\eta N_{i,1}}{N_{i,2}}\right)- \frac{1}{N_1}\sum_{i=1}^{N_1} \widetilde{t}\left(\frac{\eta N_{1,i}^{(1)}}{N_{2,i}^{(1)}}\right) \right)^2} \nonumber\\
&=\frac{1}{2N_1}\mathbb{E}{\left( \widetilde{t}\left(\frac{\eta N_{1,1}}{N_{1,2}}\right)-  \widetilde{t}\left(\frac{\eta N_{1,1}^{(1)}}{N_{2,1}^{(1)}}\right) \right)^2}  \nonumber\\
&=\frac{1}{2N}O\left(1\right)= O(\frac{1}{N}).
\end{align}
Thus, the variance proof is complete.

\section{Proof of Theorem \ref{thm:clt}}\label{sec:proof_CLT}

In this section we provide the proof of theorem \ref{thm:clt}. 
For simplicity we assume that $N_1=N_2$ and we use the notation $N:=N_1$. Also note that for simplicity we use the notation $\widehat{U_i}:=\widehat{U_i}^{(\eps)}$ Using the definition of $\Best$ we have 
\begin{align}\label{eq:CLT_1}
    \sqrt{N}\left(\Best -  \mathbb{E}\left[\Best\right]\right) &= \sqrt{N}\left(\frac{1}{2} - \frac{1}{N}\sum_{i=1}^{N}\tilde{t}\left(\widehat{U_i}\right)- \mathbb{E}\left[\frac{1}{2} - \frac{1}{N}\sum_{i=1}^{N}\tilde{t}\left(\widehat{U_i}\right)\right] \right) \nonumber\\
    & = \frac{1}{\sqrt{N}}\sum_{i=1}^N \left(\tilde{t}\left(\widehat{U_i}\right) - \mathbb{E}\left[\tilde{t}\left(\widehat{U_i}\right)\right]\right)\nonumber\\
    & = \frac{1}{\sqrt{N}}\sum_{i=1}^N \left(\tilde{t}\left(\widehat{U_i}\right) - \mathbb{E}_{\bar{i}}\left[\tilde{t}\left(\widehat{U_i}\right)\right]\right)  \nonumber\\ 
    &\qquad + \frac{1}{\sqrt{N}}\sum_{i=1}^N \left(\mathbb{E}_{\bar{i}}\left[\tilde{t}\left(\widehat{U_i}\right)\right] - \mathbb{E}\left[\tilde{t}\left(\widehat{U_i}\right)\right]\right),
\end{align}
where $\mathbb{E}_{\bar{i}}$ denotes the expectation over all samples $\bX_1,\bX_2$ except $X_{2,i}$.
In the above equation, we denote the first and second terms respectively by $S_1(\bX)$ and $S_2(\bX)$, where $\bX:=(\bX_1,\bX_2)$. In the following we prove that $S_2(\bX)$ converges to a normal random variable, and $S_1(\bX)$ converges to zero in probability. Therefore, using the Slutsky's theorem, the left hand side of \eqref{eq:CLT_1} converges to a normal random variable. 

\begin{lemma}
Let $N\to \infty$.  Then, $S_2(\bX)$ converges to a normal random variable.
\end{lemma}
\begin{proof}

Let $A_i(\bX):=\mathbb{E}_{\bar{i}}\left[\tilde{t}\left(\widehat{U_i}\right)\right] - \mathbb{E}\left[\tilde{t}\left(\widehat{U_i}\right)\right]$. Since for all $i\in \{1,...,N\}$, $A_i(\bX)$ are i.i.d. random variables, using the standard central limit theorem \citep{durrett2019probability}, $S_2(\bX)$ converges to a normal random variable. 
\end{proof}

\begin{lemma}
Let $\eps\to 0$ and $\frac{1}{\eps^dN}\to 0$. Then, $S_1(\bX)$ converges to $0$ in mean square.
\end{lemma}

\begin{proof}
In order to prove that MSE converges to zero, we need to compute the bias and variance terms separately. The bias term is obviously equal to zero since 
\begin{align}
    \mathbb{E}[S_1(\bX)] &= \mathbb{E}\left[\frac{1}{\sqrt{N}}\sum_{i=1}^N \left(\tilde{t}\left(\widehat{U_i}\right) - \mathbb{E}_{\bar{i}}\left[\tilde{t}\left(\widehat{U_i}\right)\right]\right)\right]\nonumber\\
    & = \frac{1}{\sqrt{N}}\sum_{i=1}^N \left(\mathbb{E}\left[\tilde{t}\left(\widehat{U_i}\right)\right] - \mathbb{E}\left[\tilde{t}\left(\widehat{U_i}\right)\right]\right)\nonumber\\
    &=0.
\end{align}

Next, we find an upper bound on the variance of $S_1(\bX)$ using the Efron-Stein inequality. Let $\bX':=(\bX'_1,\bX'_2)$ denote another copy of $\bX=(\bX_1,\bX_2)$ with the same distribution. We define the resampled dataset as 
\begin{align}
    \bX^{(j)}:= \left\{
    \begin{array}{cc}
         (X_{1,1},...,X_{1,j-1},X'_{1,j},X_{1,j+1},...,X_{1,N},X_{2,1},...,X_{2,N}) &  \text{if } \quad N+1\leq j\leq 2N  \\ 
         (X_{1,1},...,X_{1,N},X_{2,1},...,X_{2,j-1},X'_{2,j},X_{2,j+1},...,X_{2,N}) & \text{if } \quad 1\leq j\leq N 
    \end{array}
    \right.
\end{align} 

Let $\Delta_i=:\tilde{t}\left(\widehat{U_i}\right) - \mathbb{E}_{\bar{i}}\left[\tilde{t}\left(\widehat{U_i}\right)\right] -  \tilde{t}\left(\widehat{U_i}^{(1)}\right) + \mathbb{E}_{\bar{i}}\left[\tilde{t}\left(\widehat{U_i}^{(1)}\right)\right]$. Using the Efron-Stein inequality we can write
\begin{align}\label{eq:CLT_2}
    \mathbb{V}\left[S_1(\bX_1,\bX_2)\right] &\leq \frac{1}{2}\sum_{j=1}^{2N} \mathbb{E}\left[\left(S_1(\bX) - S_1(\bX^{(j)})\right)^2\right]\nonumber\\
    & =  N \mathbb{E}\left[\left(S_1(\bX) - S_1(\bX^{(1)})\right)^2\right]\nonumber\\
    & =  \mathbb{E}\left[\left(\sum_{i=1}^N \Delta_i\right)^2\right],\nonumber\\
    & =  \sum_{i=1}^N \mathbb{E}\left[\Delta_i^2\right]+\sum_{i\neq j} \mathbb{E}\left[\Delta_i\Delta_j\right].
\end{align}
We obtain bounds on the first and second terms in equation \eqref{eq:CLT_2}. 
First, we obtain separate bounds on $\mathbb{E}\left[\Delta_i^2\right]$ for $i=1$ and $i\neq 1$. We have
\begin{align}
    \mathbb{E}\left[\Delta_1^2\right] &= \mathbb{E}\left[\left(\tilde{t}\left(\widehat{U_i}\right) - \mathbb{E}_{\bar{i}}\left[\tilde{t}\left(\widehat{U_i}\right)\right] -  \tilde{t}\left(\widehat{U_i}^{(1)}\right) + \mathbb{E}_{\bar{i}}\left[\tilde{t}\left(\widehat{U_i}^{(1)}\right)\right]\right)^2\right]\nonumber\\
     &= \mathbb{E}\left[\left(\tilde{t}\left(\widehat{U_i}\right) - \mathbb{E}_{\bar{i}}\left[\tilde{t}\left(\widehat{U_i}\right)\right]\right)^2\right] +  \mathbb{E}\left[\left(\tilde{t}\left(\widehat{U_i}^{(1)}\right) - \mathbb{E}_{\bar{i}}\left[\tilde{t}\left(\widehat{U_i}^{(1)}\right)\right]\right)^2\right]\nonumber\\
    & \qquad - 2\mathbb{E}\left[\left(\tilde{t}\left(\widehat{U_i}\right) - \mathbb{E}_{\bar{i}}\left[\tilde{t}\left(\widehat{U_i}\right)\right] \right)\left(  \tilde{t}\left(\widehat{U_i}^{(1)}\right) - \mathbb{E}_{\bar{i}}\left[\tilde{t}\left(\widehat{U_i}^{(1)}\right)\right]\right)\right] \nonumber\\
    & \leq 4 \mathbb{E}\left[\left(\tilde{t}\left(\widehat{U_i}\right) - \mathbb{E}_{\bar{i}}\left[\tilde{t}\left(\widehat{U_i}\right)\right]\right)^2\right]\label{eq:CLT_3_1}\\
    & \leq 4 \mathbb{E}_{X_1}\left[\mathbb{E}_{X_{\bar{1}}}\left[\left(\tilde{t}\left(\widehat{U_i}\right) - \mathbb{E}_{\bar{i}}\left[\tilde{t}\left(\widehat{U_i}\right)\right]\right)^2 \Big\vert X_1=x \right]\right]\nonumber\\
    & \leq 4 \mathbb{E}_{X_1}\left[\mathbb{V}\left[\tilde{t}\left(\widehat{U_i}\right)  \right]\right]\label{eq:CLT_3_2}\\
    & \leq O(\frac{1}{N})\label{eq:CLT_3_3}.
\end{align}


Now for the case of $i\neq 1$ note that $\mathbb{E}_{\bar{i}}\left[\tilde{t}\left(\widehat{U_i}\right)\right] = \mathbb{E}_{\bar{i}}\left[\tilde{t}\left(\widehat{U_i}^{(1)}\right)\right]$. Thus, we can bound $\mathbb{E}\left[\Delta_i^2\right]$ as 
\begin{align}\label{eq:CLT_4}
    \mathbb{E}\left[\Delta_i^2\right] &= \mathbb{E}\left[\left(\tilde{t}\left(\widehat{U_i}\right) -  \tilde{t}\left(\widehat{U_i}^{(1)}\right) \right)^2\right]\nonumber\\
    & \leq O\left(\eps^d\right)\left(1-O\left(\eps^d\right)\right)O\left(\left(\frac{1}{\eps^dN}\right)^2\right) = \frac{1}{N}O\left(\frac{1}{\eps^dN}\right).
\end{align}

Hence, using \eqref{eq:CLT_3_3} and \eqref{eq:CLT_4} we get
\begin{align}
    \sum_{i=1}^N \mathbb{E}\left[\Delta_i^2\right] \leq O\left(\frac{1}{\eps^dN}\right).
\end{align}

Note that we can similarly prove that 
the bound $\sum_{i\neq j} \mathbb{E}\left[\Delta_i\Delta_j\right]\leq O\left(\frac{1}{\eps^dN}\right)$. Thus, from equation \eqref{eq:CLT_2} we have $\mathbb{V}\left[S_1(\bX_1,\bX_2)\right]\leq O\left(\frac{1}{\eps^dN}\right)$, which convergence to zero if the assumption $\frac{1}{\eps^dN}\to 0$ holds. 

\end{proof}


\section{Proof of Theorem~\ref{thm:ensemble_theorem}}\label{section:proofofensemble_theorem}

First note that since $N_{1,1}$ and $N_{2,1}$ are independent we can write 
\begin{align}\label{eq:E_ratio}
\mathbb{E}\left[\frac{N_{1,i}}{N_{2,i}}\middle\vert X_{2,i}\right] &= \mathbb{E}\left[N_{1,i}\middle\vert X_{2,i}\right] \mathbb{E}\left[{N^{-1}_{2,i}}\middle\vert X_{2,i}\right].
\end{align}
From (37) and (38) of \citep{noshad2018hash} we have

\begin{align}
    \mathbb{E}\left[N_{1,i}\right] &=N_1\epsilon^d\left[f_1(X_{2,i})+ \sum_{l=1}^q C_l(X_{2,i})\epsilon^{l}+O\left(C_q(X_{2,i})\epsilon^{q}\right)\right],\label{eq:ens_proof_1}\\
    \mathbb{E}\left[(N_{2,i})^{-1}\right] &=N_{2}^{-1}\epsilon^{-d}\left[f_2(X_{2,i})+ \sum_{l=1}^q C_l(X_{2,i})\epsilon^{l}+O\left(C_q(X_{2,i})\epsilon^{q}\right)\right]^{-1}\left(1+O\left(\frac{1}{N_2\epsilon^df_2(X_{2,i})}\right)\right),\label{eq:ens_proof_2}
\end{align}
where $C_i(x)$ for $1\leq i\leq q$ are functions of $x$.
Plugging equations \eqref{eq:ens_proof_1} and \eqref{eq:ens_proof_2} into \eqref{eq:E_ratio} results in 

\begin{align}
    \mathbb{E}\left[\frac{\eta N_{1,i}}{N_{2,i}}\middle\vert X_{2,i}\right] &= \frac{f_1(X_{2,i})}{f_2(X_{2,i})} +\sum_{i=1}^q C''_i\epsilon^{i}+O\left(\frac{1}{N\epsilon^d}\right),
\end{align}
where $C''_1,...,C''_q$ are constants. 

Now apply the ensemble theorem (\citep{moon2018ensemble}, Theorem 4). Let $\mathcal{T}:=\{t_1,...,t_T\}$ be a set of index values with $t_i<c$, where $c>0$ is a constant. Define $\epsilon(t):=tN^{-1/2d}$. According to the ensemble theorem in (\citep{moon2018ensemble}, Theorem 4) if we choose the parameters $\psi_i(t)=t^{i/d}$ and $\phi'_{i,d}(N)=\phi_{i,\kappa}(N)/N^{i/d}$, the following weighted ensemble converges to the true value with the MSE rate of $O(1/N)$:
\begin{align}\label{def_Ui_w}
    \widehat{U}_i^{\bw}:=\sum_{l=1}^{L} w_{l}\widehat{U}_i,
\end{align}
where the weights $w_l$ are the solutions of the optimization problem in equation \eqref{eq:optimization}. Thus, the bias of the ensemble estimator can be written as

\begin{align}
\mathbb{E}_{\bar{X_{i}}}\left[\left.\widehat{U}_i^{\bw}\right|X_{2,i}\right]&=\frac{f_1(X_{2,i})}{f_2(X_{2,i})} + O(1/\sqrt{N_1}).\label{eq:weightedestimator}
\end{align}

By Lemma 4.4 in \citep{noshad2017direct} and the fact that function $t(x):=|p_1 x-p_2|-p_1x$ is Lipschitz continuous with constant $2p_1$,
\begin{align}
&\left|\mathbb{E}_{\bar{X_{i}}}[t(\widehat{U}_i^{\bw})|X_{2,i}]-t\left(\frac{f_1(X_{2,i})}{f_2(X_{2,i})}\right)\right|\le 2p_1\left( \sqrt{\mathbb{V}_{\bar{X_{i}}}[\widehat{U}_i^{\bw}|X_{2,i}]}+\left|\mathbb{B}_{\bar{X_{i}}}[\widehat{U}_i^{\bw}|X_{2,i}]\right|\right).
\end{align}
Here $\mathbb{B}$ and $\mathbb{V}$ represent bias and variance, respectively. By \eqref{eq:weightedestimator}, we have $\mathbb{B}_{\bar{X_{i}}}[\widehat{U}_i^{\bw}|X_{2,i}]=O(1/\sqrt{N_1})$; and by Theorem 2.2 in \citep{noshad2017direct},  $\mathbb{V}_{\bar{X_{i}}}[\widehat{U}_i^{\bw}|X_{2,i}]=O(1/N_1).$ Thus,
\begin{align}
\mathbb{E}_{\bar{X_{i}}}[t(\widehat{U}_i^{\bw})|X_{2,i}]-t\left(\frac{f_1(X_{2,i})}{f_2(X_{2,i})}\right)=O(1/\sqrt{N_1}).
\end{align}
So the bias of the estimator $\mathcal{F}(\bX_1,\bX_2)$ is given by
\begin{align}
\mathbb{B}(\mathcal{F}(\bX_1,\bX_2))&=\left|\mathbb{E}_{\bX_1,\bX_2}\left[ \frac{1}{2N_2}\sum_{i=1}^{N_2} t(\widehat{U}_i^{\bw})\right]-\frac{1}{2}\mathbb{E}_{X_{2,i}}\left[t\left(\frac{f_1(X_{2,i})}{f_2(X_{2,i})}\right)\right]\right|\nonumber\\
&=\frac{1}{2N_2}\sum_{i=1}^{N_2}\left|\mathbb{E}_{X_{2,i}}\left[\mathbb{E}_{\bar{X_{i}}}[t(\widehat{U}_i^{\bw})|X_{2,i}]-t\left(\frac{f_1(X_{2,i})}{f_2(X_{2,i})}\right)\right]\right|= O(1/\sqrt{N_1}).
\end{align}
Finally, since the variance of $\widehat{U}_i^{\bw}$ can easily be upper bounded by $O(1/N)$ using the Efron-Stein inequality using the same steps in Appendix. \ref{sec:app_binary_convergence}.

%


\section{Proof of Theorem \ref{thm:chebyshev_optimization}} \label{sec:chebyshevproof}

In order to prove the theorem we first prove that the solutions of the constraint in \eqref{eq:optimization} for $t_i=s_i$ can be written as a function of the shifted Chebyshev polynomials. Then we find the optimal solutions of $w_i$ which minimize $\|w\|_2^2$.

\begin{lemma}
All solutions of the constraint

\begin{align}\label{eq:orig_const}
    &\sum_{k=0}^{L-1} \omega _ { k } s _ { k } ^ { j } = 0 ,\quad \forall j \in \{1,...,d\}\nonumber\\
    &\sum _ {k=0}^{L-1} \omega _ { k } = 1,
\end{align}
have the following form
\begin{align}\label{eq:wi_solutions}
    w_i = \sum_{k=0}^{d}\frac{2T^\alpha_k(0)}{L} T^\alpha_k(s_i) + \sum_{k=d+1}^{L-1}c_k T^\alpha_k(s_i) - \frac{1}{L}  \qquad \forall i \in \{0,...,L-1\},
\end{align}
for some $c_k\in \mathbb{R}$, $ k\in \{d+1,...,L-1\}$, and for any $c_k\in \mathbb{R}$, $ k\in \{d+1,...,L-1\}$, $w_i$ given by \eqref{eq:wi_solutions} satisfy the equations in \eqref{eq:orig_const}.

\end{lemma}

\begin{proof}

We can rewrite \eqref{eq:orig_const} as

\begin{align}\label{eq:Chebyshev_1}
    &\sum _ { j=0 }^d \sum _ {k=0}^{L-1} \omega _ { k } x _ { j } s _ { k } ^ {j} = x _ { 0 } \quad \forall x _ { j } \in \mathbb { R }.
\end{align}

Note that setting $\forall i \in \{1,...,d\}, x_i=0$ in \eqref{eq:Chebyshev_1} yields the second constraint in \eqref{eq:optimization}, and $\forall i\neq j, x_i=0$ results in the first set of $d$ constraints in \eqref{eq:optimization}.
Using the fact that $\sum _ { j } \sum _ { k } \omega _ { k } x _ { j } s _ { k } ^ { j } = \sum _ { k } \omega _ { k } \sum _ { j }  x _ { j } s _ { k } ^ { j } $ we can equivalently write the constraint as

\begin{align}\label{eq:Chebyshev_2}
    \sum _ { k=0 }^{L-1} \omega _ { k } f \left( s _ { k } \right) = f ( o ) \quad \forall f \in P _ { d },
\end{align}
where $P_d$ is the family of the polynomials of degree $d$. One can expand the polynomial $f(x)\in P_d$ defined in $[0,\alpha]$ in the Chebyshev polynomial basis:

$$f(x)=\sum_{i=0}^dr_iT^\alpha_i(x).$$

Thus, we can write the constraint in \eqref{eq:Chebyshev_2} as 

\begin{align}
    &\sum _ { k=0 }^{L-1} \omega _ { k }\sum_{j=0}^d r_j T^\alpha _ { j } \left( s _ { k } \right) = \sum_{j=0}^d r_j T^\alpha _ { j } ( 0 ) \quad \forall r_j\in \mathbb{R},
\end{align}
which can be further formulated as
\begin{align}
    &\sum_{j=0}^d r_j \sum _ { k=0 }^{L-1} \omega _ { k } T^\alpha _ { j } \left( s _ { k } \right) = \sum_{j=0}^d r_j T^\alpha _ { j } ( 0 ) \quad \forall r_j\in \mathbb{R},
\end{align}
which is equivalent to the following constraint in the Chebyshev polynomials basis:

\begin{align}\label{eq:Chebyshev_3}
    &\sum _ { k=0 }^{L-1} \omega _ { k } T^\alpha _ { j } \left( s _ { k } \right) = T^\alpha _ { j } ( 0 ) \quad \forall j \in \{0,...,d\}.
\end{align}

Now we use the Chebyshev polynomial approximation method in order to simplify the optimization problem in equation \eqref{eq:optimization}. 
Define a function $f:[0,\alpha]\to\mathbb{R}$ such that $f(s_i)=w_i, i\in \{0,...,L-1\}$.

We can write $f(x)$ in terms of Chebyshev interpolation polynomials with the $L$ points $0<s_0,...,s_{L-1}<1$ as

\begin{align}\label{eq:plnm_approx}
    f(x) = \sum_{k=0}^{L-1}c_k T^\alpha_k(x) - \frac{c_0}{2}  + R(x),
\end{align}
where $R(x)$ is the error of approximation and is given by

\begin{align}
    R(x) = \frac{f^{(L)}(\xi)}{L!} \prod_{j=0}^{L-1}\left(x-s_{j}\right),
\end{align}
for some $\xi \in [0,\alpha]$. Thus we have

\begin{align}\label{eq:Chb_basis}
    w_i = f(s_i) = \sum_{k=0}^{L-1}c_k T^\alpha_k(s_i) - \frac{c_0}{2}  \qquad \forall i \in \{0,...,L-1\}.
\end{align}

The interpolation coefficients in \eqref{eq:plnm_approx} can be computed as follows

\begin{equation}\label{eq:Chb_coeff}
c_{k}=\frac{2}{L} \sum_{j=0}^{L-1} f\left(s_{j}\right) T^\alpha_{k}\left(s_{j}\right) \qquad \forall k \in \{0,...,L-1\}.
\end{equation}

Comparing the equation \eqref{eq:Chb_coeff} with the constraint in \eqref{eq:Chebyshev_3} we get

\begin{equation}\label{eq:coeff_equations}
    c_k = \frac{2T^\alpha_k(0)}{L} \quad \forall k\in \{0,...,d\}.
\end{equation}

Thus, we can write equation \eqref{eq:Chb_basis} as 

\begin{align}\label{eq:Chb_basis_2}
    w_i = f(s_i) = \sum_{k=0}^{d}\frac{2T^\alpha_k(0)}{L} T^\alpha_k(s_i) + \sum_{k=d+1}^{L-1}c_k T^\alpha_k(s_i) - \frac{1}{L}  \qquad \forall i \in \{0,...,L-1\}.
\end{align}

Next, for any $c_k\in \mathbb{R}$, $ k\in \{d+1,...,L-1\}$, $w_i$ given by \eqref{eq:wi_solutions} satisfy equation \eqref{eq:Chebyshev_3}, which is an equivalent form of the original constraints in equation \eqref{eq:orig_const}. Using \eqref{eq:Chb_basis_2} we can write:

\begin{align}\label{eq:const_equiv_check}
    \sum _ { i=0 }^{L-1} \omega_i T^\alpha _ { j } \left( s_i \right) &= \sum _ { i=0 }^{L-1} T^\alpha_{j}(s_i) \left[\sum_{k=0}^{d}\frac{2T^\alpha_k(0)}{L} T^\alpha_k(s_i) + \sum_{k=d+1}^{L-1}c_k T^\alpha_k(s_i) - \frac{c_0}{2} \right]\nonumber\\
    &=  \sum_{k=0}^{d}\frac{2T^\alpha_k(0)}{L} \sum_{i=0 }^{L-1} T^\alpha_{j}(s_i) T^\alpha_k(s_i) + \sum_{k=d+1}^{L-1}c_k \sum_{i=0 }^{L-1} T^\alpha_{j}(s_i)T^\alpha_k(s_i) - \sum_{i=0 }^{L-1} T^\alpha_{j}(s_i)\frac{T^\alpha_0(s_i)}{L},
\end{align}
where for the last term we have used the fact that $c_0=\frac{2T^\alpha_0(0)}{L}=\frac{2T^\alpha_0(s_i)}{L}=\frac{2}{L}$ from equation \eqref{eq:coeff_equations}. Now in order to simplify equation \eqref{eq:const_equiv_check}, we use the orthogonality property of the Chebyshev (and shifted Chebyshev) polynomials. That is, if $s_i$ are the zeros of $T_L^*$, then 
\begin{align}\label{eq:orthogonality}
    \sum_{i=0 }^{L-1} T^\alpha_{j}(s_i)T^\alpha_k(s_i) = K_j\delta_{kj},
\end{align}
where $K_j=L$ for $j=0$ and $K_j=L/2$ for $L-1\geq j>0$. Hence, \eqref{eq:const_equiv_check} simplifies to 
\begin{align}\label{eq:const_equiv_check_2}
    \sum _ { i=0 }^{L-1} \omega_i T^\alpha _ { j } \left( s_i \right) &= \sum_{k=0}^{d}\frac{2T^\alpha_k(0)}{L} K_j\delta_{kj} + \sum_{k=d+1}^{L-1}c_k K_j\delta_{kj} - K_0\delta_{0j}\frac{1}{L}.
\end{align}
Thus, for $j=0$ we get 
\begin{align}\label{eq:const_equiv_check_3}
    \sum _ { i=0 }^{L-1} \omega_i T^\alpha _ { j } \left( s_i \right) &= 2T^\alpha_0(0)-1 = T^\alpha_0(0),
\end{align}
and for $d\geq j>0$ we get 
\begin{align}\label{eq:const_equiv_check_4}
    \sum _ { i=0 }^{L-1} \omega_i T^\alpha _ { j } \left( s_i \right) = T^\alpha_j(0),
\end{align}
which shows that $w_i$ satisfy the constraint in equation \eqref{eq:Chebyshev_3}, which is an equivalent form of the original constraints in equation \eqref{eq:orig_const}. The proof of the lemma is complete. 
\end{proof}

\textbf{Proof of Theorem \ref{thm:chebyshev_optimization}: }
In \eqref{eq:wi_solutions}, $c_k$, $k\in \{d+1,...,L-1\}$ will be determined such that the term $\|w\|_2^2$ in the original optimization problem is minimized. 
Using \eqref{eq:wi_solutions}, the objective function of the optimization problem in \eqref{eq:optimization} can be simplified as

\begin{align}\label{eq:norm_expansion}
    \|w\|_2^2 &= \sum_{i=0}^{L-1}w_i^2\nonumber\\ 
     &= \sum_{i=0}^{L-1} f(s_i)^2\nonumber\\
    &=\sum_{i=0}^{L-1} A_{i}^2 +\sum_{i=0}^{L-1} 2A_i\sum_{k=d+1}^{L-1}c_kT^\alpha_k(s_i) + \sum_{i=0}^{L-1} \left(\sum_{k=d+1}^{L-1}c_k T^\alpha_k(s_i)\right)^2
\end{align}
where $A_{i}:= \sum_{k=0}^{d}\frac{2T_k^*(0)}{L} T^\alpha_k(s_i) - \frac{1}{L} $. Note that since the first term in \eqref{eq:norm_expansion} is constant, the minimization of $\|w\|_2^2$ is equivalent to minimization of the following quadratic expression in terms of the variables $\{c_{d+1},...,c_{L-1}\}$:

\begin{align} \label{eq:Chb_opt}
    G(c_{d+1},...,c_{L-1}) := \sum_{i=0}^{L-1} 2A_i\sum_{k=d+1}^{L-1}c_kT^\alpha_k(s_i) + \sum_{i=0}^{L-1} \left(\sum_{k=d+1}^{L-1}c_kT^\alpha_k(s_i)\right)^2.
\end{align}

We first show that the first term in \eqref{eq:Chb_opt} is equal to zero.

\begin{align}
    \sum_{i=0}^{L-1} 2A_i\sum_{k=d+1}^{L-1}c_kT^\alpha_k(s_i) &= \sum_{i=0}^{L-1} 2\left(\sum_{k=0}^{d}\frac{2T_k^*(0)}{L} T^\alpha_k(s_i) - \frac{1}{L} \right)\sum_{k=d+1}^{L-1}c_kT^\alpha_k(s_i)\nonumber\\ 
    &=\frac{2}{L}\sum_{i=0}^{L-1} \sum_{k=0}^{d}\sum_{j=d+1}^{L-1}2T_k^*(0) T^\alpha_k(s_i)c_jT^\alpha_j(s_i) - \sum_{i=0}^{L-1}\sum_{j=d+1}^{L-1}c_jT^\alpha_j(s_i) \nonumber\\ 
    &=\frac{2}{L} \sum_{k=0}^{d}\sum_{j=d+1}^{L-1}2T_k^*(0) c_j\sum_{i=0}^{L-1} T^\alpha_k(s_i)T^\alpha_j(s_i) - \sum_{j=d+1}^{L-1}c_j\sum_{i=0}^{L-1}T^\alpha_j(s_i)T^\alpha_0(s_i) \nonumber\\ 
    &= 0.
\end{align}

Note that in the third line, we have used the identity $T_0^*(s_i)=1$. In the fourth line we have used the orthogonality identity \eqref{eq:orthogonality}. 
Finally, setting $c_{d+1} =...= c_{L-1}=0$ minimizes the second term and as a result $G(c_{d+1},...,c_{L-1})$. Thus, the optimal solutions of $w_i$ are given as
\begin{align}
    w_i = \frac{2}{L}\sum_{k=0}^{d}T^\alpha_k(0) T^\alpha_k(s_i) - \frac{1}{L}  \qquad \forall i \in \{0,...,L-1\},
\end{align}
which completes the proof.

\section{Proof of Theorem~\ref{multi-crlassestimatortheorem}}\label{section:proofofmulti-classestimatortheorem}
\textbf{Bias proof: }In the following we state a multivariate generalization of Lemma 3.2 in \citep{noshad2017direct}.

\begin{lemma}
Assume that $g(x_1,x_2,\ldots,x_k): \mathcal{X}\times \cdots \times \mathcal{X} \to \mathbb{R}$ is Lipschitz continuous with constant $H_g > 0$, with respect to $x_1,\ldots,x_k$. If $\widehat{T}_i$ where $0\leq i \leq k$ be random variables, each one with a variance $\mathbb{V}[\widehat{T}_i]$ and a bias with respect to given constant values $T_i$, defined as $\mathbb{B}[\widehat{T}_i]:=T_i - \mathbb{E}[\widehat{T}_i]$, then the bias of $g(\widehat{T}_1,\ldots,\widehat{T}_k)$ can be upper bounded by
\begin{align}\label{multi_bias_bound}
&\left|\mathbb{E}\left[g(\widehat{T}_1,\ldots,\widehat{T}_k)-g(T_1,\ldots,T_k)\right]\right| \leq H_g \sum_{i=1}^k  \left(\sqrt{\mathbb{V}[\widehat{T}_i]} + \left|\mathbb{B}[\widehat{T}_i]\right|\right).
\end{align}
{\bf Proof:}
\end{lemma}
\begin{align}
\left|\mathbb{E}\left[g(\widehat{T}_1,\ldots,\widehat{T}_\lambda)-g(T_1,\ldots,T_\lambda)\right]\right|  &\leq \sum_{i=1}^\lambda \left|\mathbb{E}\left[g(\widehat{T}_1,\ldots,\widehat{T}_{i},T_{i+1},\ldots,T_\lambda)-g(T_1,\ldots,T_\lambda)\right]\right| \nonumber\\
&\leq \sum_{i=1}^\lambda  H_g \left(\sqrt{\mathbb{V}[\widehat{T}_i]} + \left|\mathbb{B}[\widehat{T}_i]\right|\right),
\end{align}
where in the last inequality we have used Lemma 3.2 in \citep{noshad2017direct}, by assuming that $g$ is only a function of $\widehat{T}_i$.

Now, we plug $\widehat{U}^{\bw}_i$ defined in \eqref{def_Ui_w} into $\widehat{T}_{i}$ in \eqref{multi_bias_bound}. Using equation \eqref{eq:weightedestimator} and the fact that $\mathbb{V}_{\bar{X_{i}}}[\widehat{U}^{\bw}_i|X_{2,i}]=O(1/N_1)$ (as mentioned in Appendix \ref{section:proofofensemble_theorem}), concludes the bias proof.

\textbf{Variance proof:}
Without loss of generality, we assume that $N_\lambda =\max(N_1,N_2,\ldots,N_\lambda )$. We consider $(N_\lambda -N_l)$ virtual random nodes $X_{l,N_l+1},\ldots,X_{l,N_\lambda }$ for $1\le l\le \lambda -1$ which follow the same distribution as $X_{l,1},\ldots,X_{l,N_l}$. Let $Z_i:=(X_{1,i},X_{2,i},\ldots,X_{\lambda ,i})$. Now we consider $\bZ:=(Z_1,\ldots,Z_{N_\lambda })$ and another independent copy of $\bZ$ as  $\bZ':=(Z'_1,\ldots,Z'_{N_\lambda })$, where $Z_i:=(X_{1,i}',X_{2,i}',\ldots,X_{\lambda ,i}')$. Let $\bZ^{(i)}:=(Z_1,\ldots,Z_{i-1},Z'_i,Z_{i+1},\ldots,Z_{N_\lambda })$ and $\mathcal{E}_k(\bZ):=\mathcal{E}_k(\bX_1,\bX_2,\ldots,\bX_\lambda )$. 
Let
\begin{align}
 &B_{\alpha,i} :=
\tilde{t}\bigg(\widehat{U}^{\bw}_{(1/\lambda)}(X_{\lambda,i}),\widehat{U}^{\bw}_{(2/\lambda)}(X_{\lambda,i}),\ldots,\widehat{U}^{\bw}_{((\lambda-1)/\lambda))}(X_{\lambda,i})\bigg)\nonumber\\
&\qquad \qquad - \tilde{t}\bigg(\widehat{U}^{\bw}_{(1/\lambda)}(X_{\lambda,i}'),\widehat{U}^{\bw}_{(2/\lambda)}(X_{\lambda,i}'),\ldots,\widehat{U}^{\bw}_{((\lambda-1)/\lambda)}(X_{\lambda,i}')\bigg).
\end{align}
We have
\begin{align}\label{variance_main}
&\frac{1}{2} \sum_{i=1}^{N_\lambda } \mathbb{E}\left[\left(\mathcal{E}_k(\bZ)-\mathcal{E}_k(\bZ^{(i)})\right)^2\right]
=\frac{1}{2N_\lambda }\mathbb{E}{\left[\sum_{i=1}^{N_\lambda } B_{\alpha,i} \right]^2} \nonumber\\
&\qquad \qquad =\frac{1}{2N_\lambda }\sum_{i=1}^{N_\lambda } \mathbb{E}[B_{\alpha,i}^2]+ \frac{1}{2N_\lambda }\sum_{i\neq j} \mathbb{E}[B_{\alpha,i}B_{\alpha,j}]=\frac{1}{2} \mathbb{E}[B_{\alpha,2}^2]+ \frac{N_\lambda }{2} \mathbb{E}[B_{\alpha,2}]^2.
\end{align}
The last equality follows from $\mathbb{E}[B_{\alpha,i}B_{\alpha,j}]=\mathbb{E}[B_{\alpha,i}]\mathbb{E}[B_{\alpha,j}]=\mathbb{E}[B_{\alpha,i}]^2$ for $i\neq j$. With a parallel argument in the proof of Lemma 4.10 in \citep{noshad2017direct}, we have
\begin{equation}
\mathbb{E}[B_{\alpha,2}]=O\left(\frac{\lambda }{N_\lambda }\right)\textrm{ and }\mathbb{E}[B_{\alpha,2}^2]=O\left(\frac{\lambda ^2}{N_\lambda }\right).
\end{equation}
Then applying Efron-Stein inequality, we obtain
\begin{align} \label{ESTa_Var_proof_1}
\mathbb{V}[\mathcal{E}_k(\bZ)] &\leq \frac{1}{2} \sum_{i=1}^M \mathbb{E}\left[\left(\mathcal{E}_k(\bZ)-\mathcal{E}_k(\bZ^{(i)})\right)^2\right]=O\left(\frac{\lambda ^2}{N_\lambda }\right).
\end{align}
Since the ensemble estimator is a convex combination of some single estimators, the proof is complete.

\section{Supplementary Numerical Results}\label{sec:supp_num_results}

In this section we perform extended experiments on the proposed benchmark learner. We perform experiments on different simulated datasets with Gaussian, beta, Rayleigh and concentric distributions of various dimensions of up to $d=100$. 

Figure \ref{fig:optimal_weights} represents the scaled coefficients of the base estimators and their corresponding weights in the ensemble estimator using the arithmetic and Chebyshev nodes for (a) $d=10$ ($L=11$) and (b) $d=100$ ($L=101$). The optimal weights for the arithmetic nodes decreases monotonically. However, the optimal weights for the Chebyshev nodes has an oscillating pattern. 

\begin{figure}
  \centering
  \subfigure[$d=10$]{\centering\qquad\qquad\quad\includegraphics[width=0.78\textwidth]{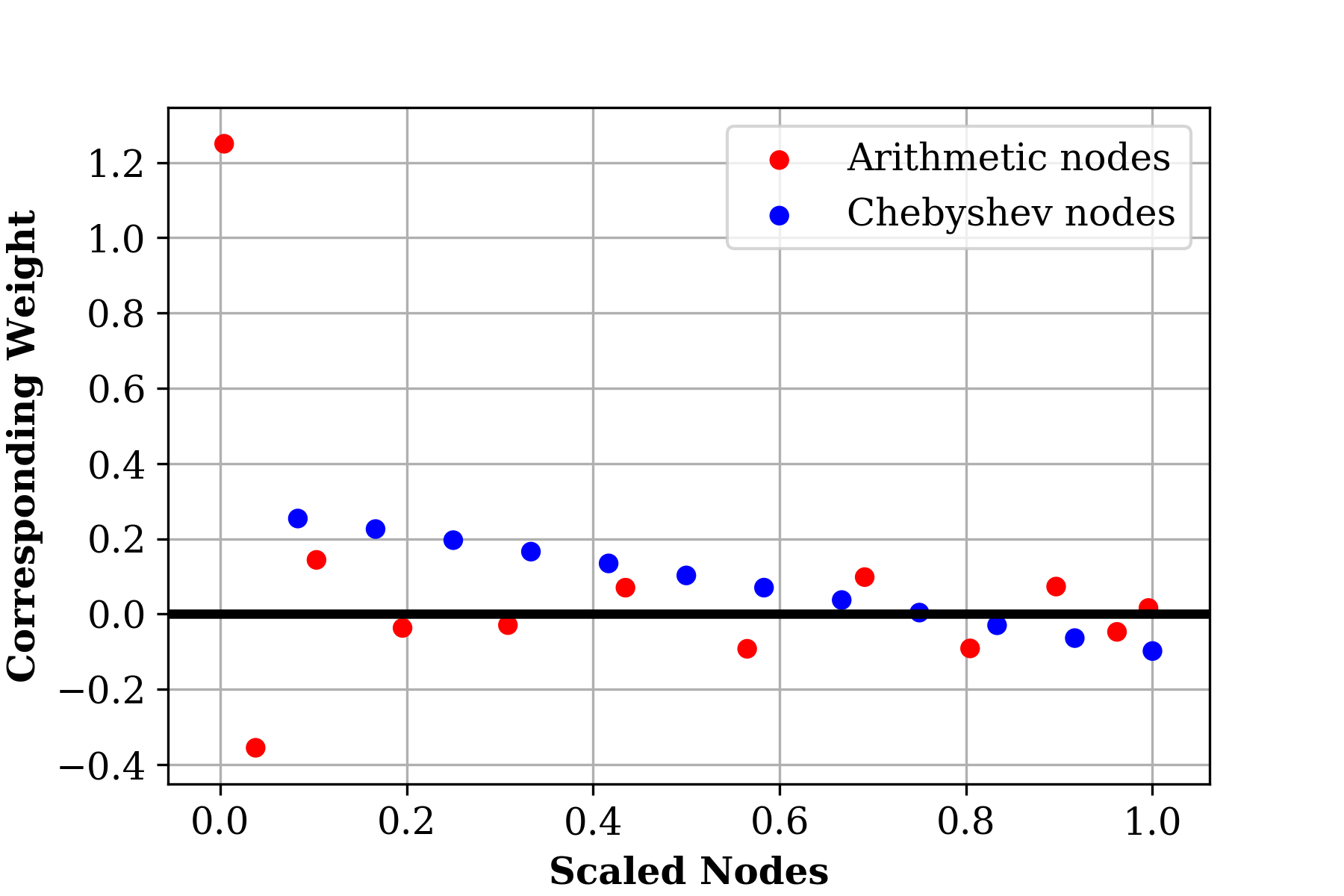}}\newline
  \subfigure[$d=100$]{\centering\includegraphics[width=0.7\textwidth]{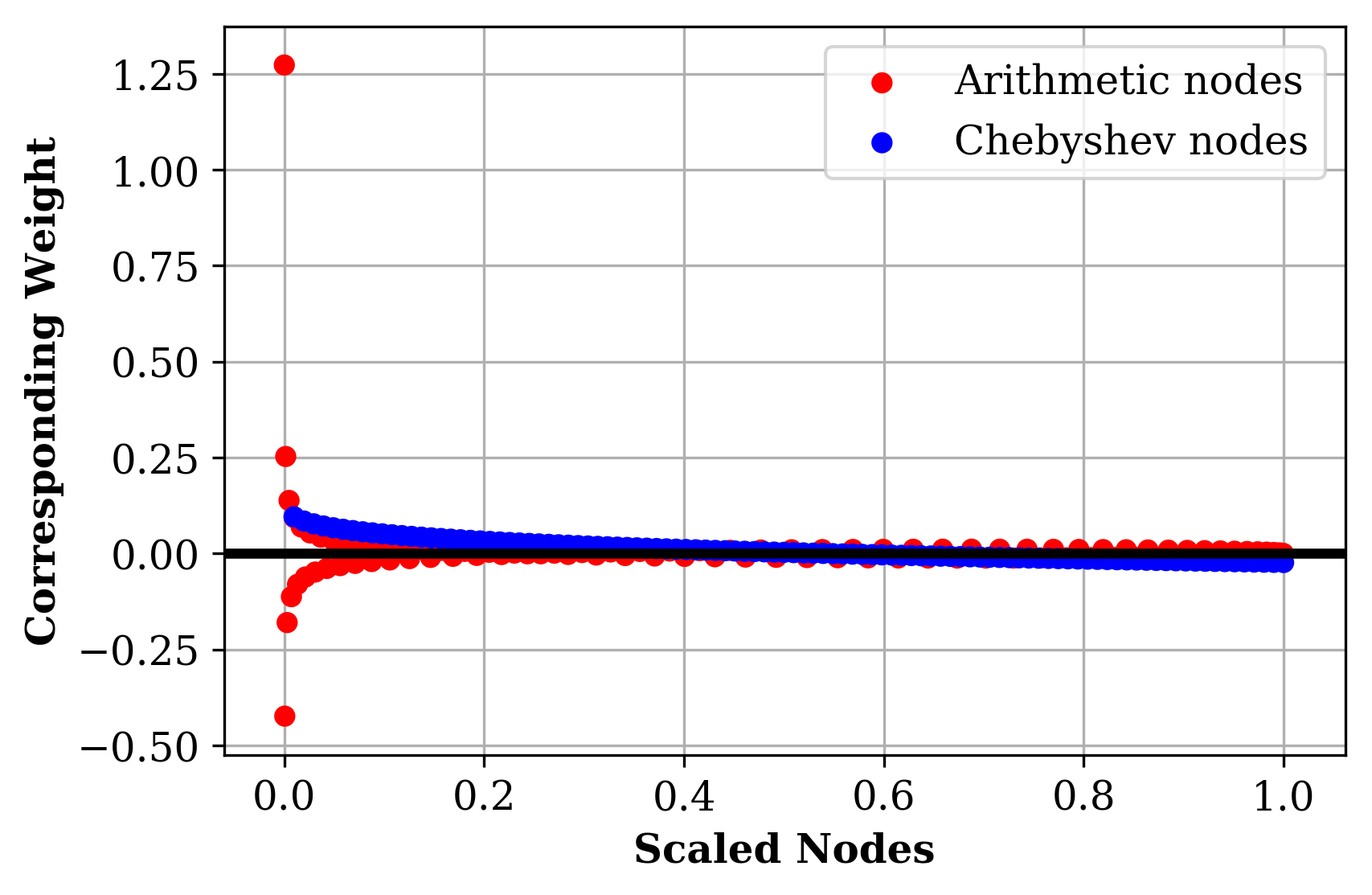}}
\caption{The scaled coefficients of the base estimators and their corresponding optimal weights in the ensemble estimator using the arithmetic and Chebyshev nodes for (a) $d=10$ and (b) $d=100$. The optimal weights for the arithmetic nodes decreases
monotonically. However, the optimal weights for the Chebyshev nodes has an oscillating
pattern.}
\label{fig:optimal_weights}
\end{figure}
In Figures \ref{fig:cmp_estimator_d10} and \ref{fig:cmp_estimator_d100} we consider binary classification problems respectively with $4$-dimensional and $100$-dimensional isotropic normal distributions with covariance matrix $\sigma \mathbf{I}$, where the means are separated by 2 units in the first dimension.
We plot the Bayes error estimates for different methods of Chebyshev, arithmetic and uniform weight assigning methods for different sample sizes, in terms of (a) MSE rate and (b) mean estimates with $\%95$ confidence intervals. Although both the Chebyshev and arithmetic weight assigning methods are asymptotically optimal, in our experiments the benchmark learner with Chebyshev nodes has a better convergence rate for finite number of samples. For example in Figures \ref{fig:cmp_estimator_d10} and \ref{fig:cmp_estimator_d100}, for $1600$ samples, MSE of the Chebyshev method is respectively $\%10$ and $\%92$ less than MSE of the arithmetic method.
\begin{figure}
  \centering
  \subfigure[Mean square error]{\centering\qquad\quad\includegraphics[width=0.7\textwidth]{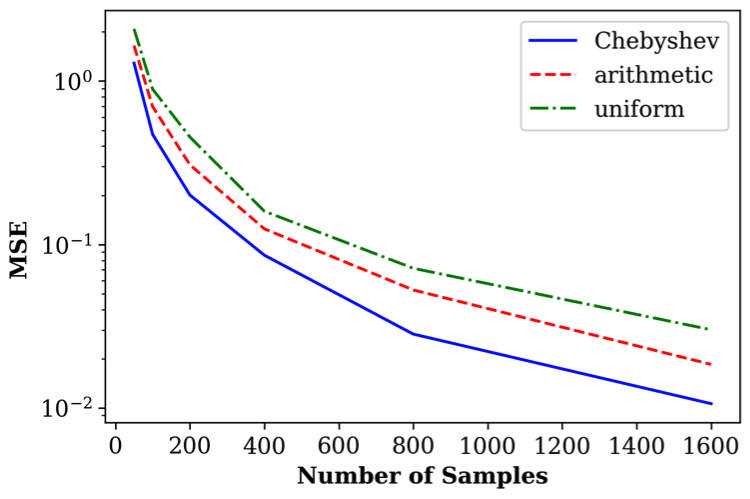}}\newline
  \subfigure[Mean estimates with $\%95$ confidence intervals]{\centering\includegraphics[width=0.7\textwidth]{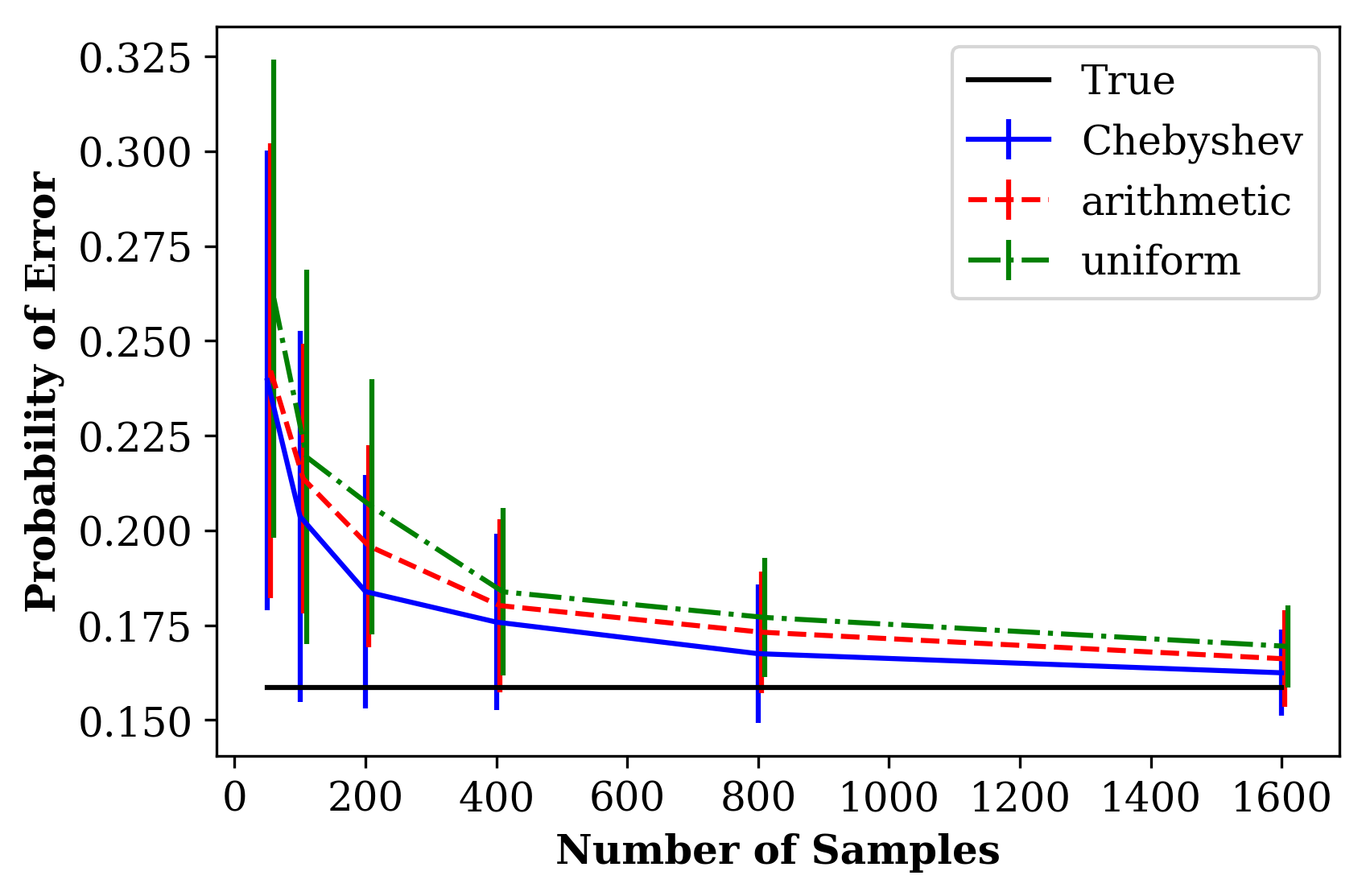}}
\caption{Comparison of the Bayes error estimates for different methods of Chebyshev, arithmetic and uniform weight assigning methods for a binary classification problem with 4-dimensional isotropic normal distributions. The Chebyshev method provides a better convergence rate.}
\label{fig:cmp_estimator_d10}
\end{figure}
\begin{figure}
  \centering
  \subfigure[Mean square error]{\centering\qquad\quad\includegraphics[width=0.63\textwidth]{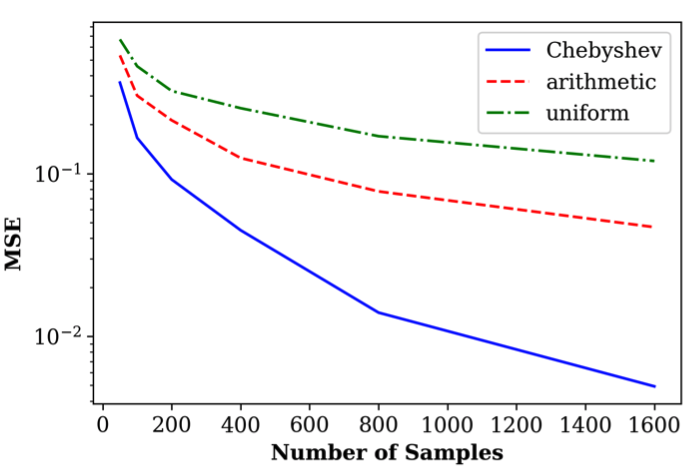}}\newline
  \subfigure[Mean estimates with $\%95$ confidence intervals]{\centering\includegraphics[width=0.7\textwidth]{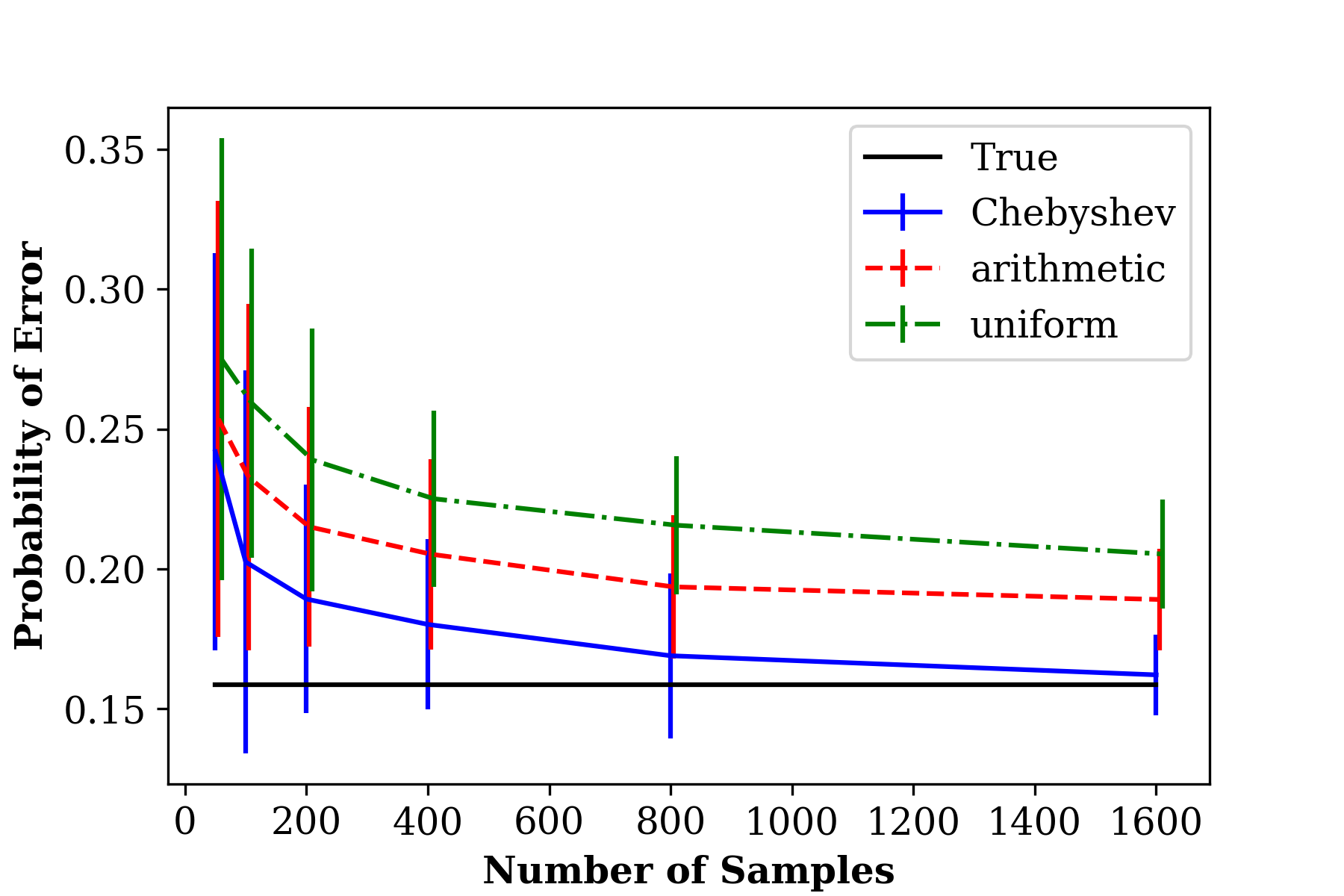}}
\caption{Comparison of the Bayes error estimates for different methods of Chebyshev, arithmetic and uniform weight assigning methods for a binary classification problem with $100$-dimensional isotropic normal distributions. The Chebyshev method provides a better convergence rate compared to the arithmetic and uniform methods.}
\label{fig:cmp_estimator_d100}
\end{figure}

In Figures \ref{fig:com_alpha_normal} (a) and (b) we compare the Bayes error estimates for ensemble estimator with Chebyshev nodes with different scaling coefficients $\alpha = 0.1,0.3,0.5,1.0$ for binary classification problems respectively with $10$-dimensional and $50$-dimensional isotropic normal distributions with covariance matrix $2 \mathbf{I}$, where the means are separated by 5 units in the first dimension.

\begin{figure}
  \centering
  \subfigure[Mean square error]{\centering\qquad\quad\includegraphics[width=0.63\textwidth]{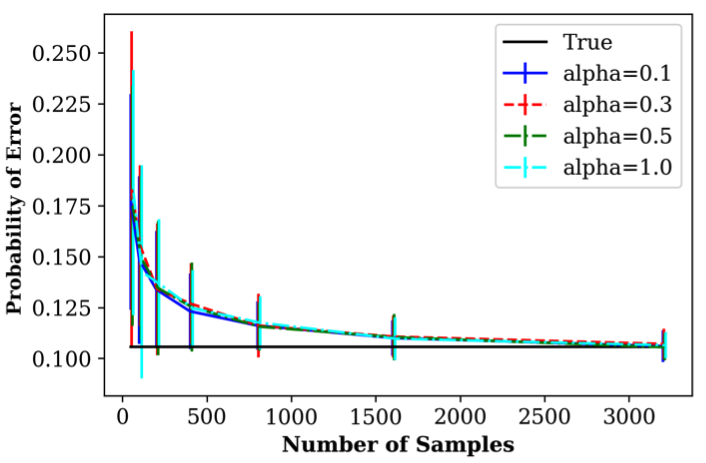}}\newline
  \subfigure[Mean estimates with $\%95$ confidence intervals]{\centering\includegraphics[width=0.7\textwidth]{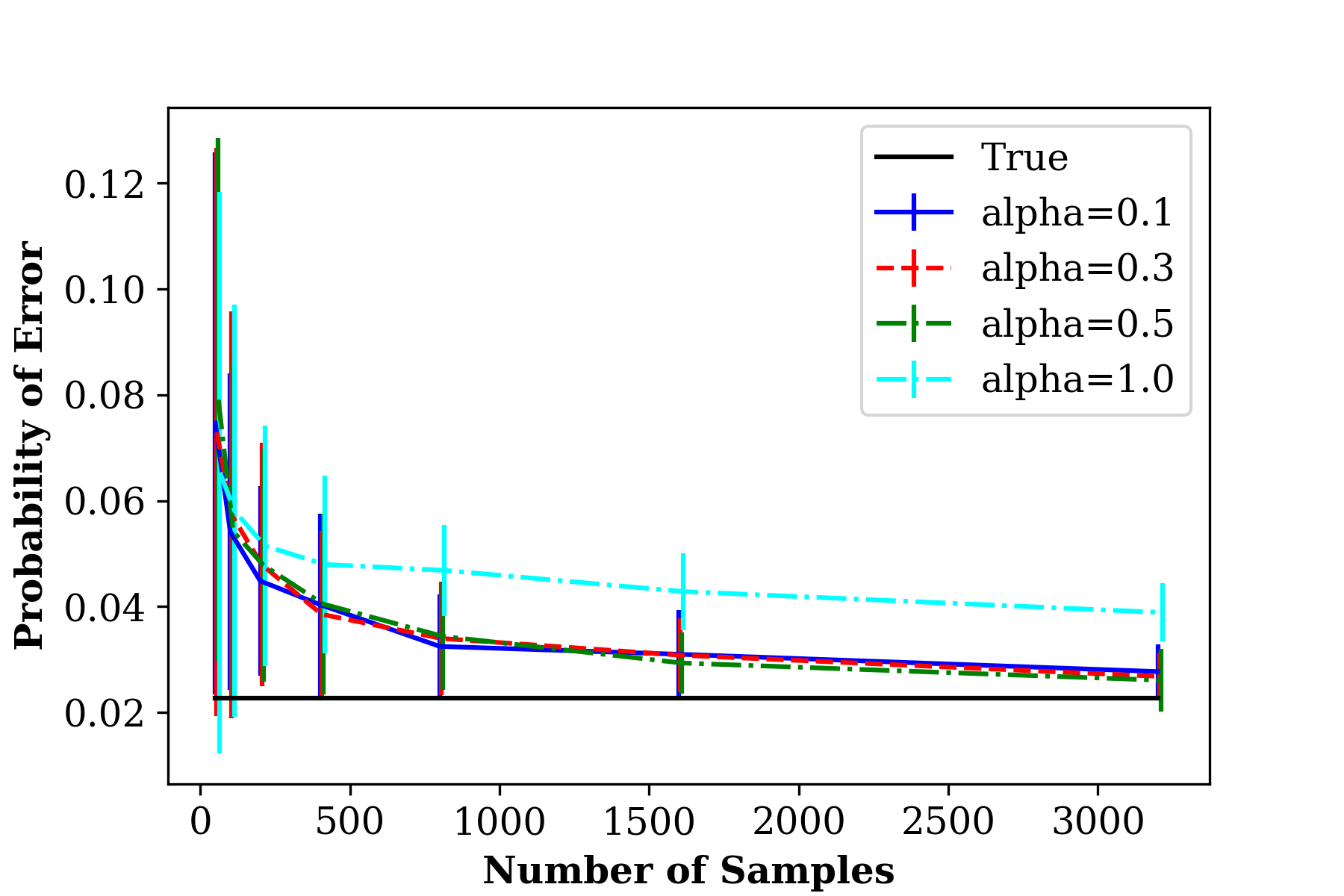}}
\caption{Comparison of the Bayes error estimates for ensemble estimator with Chebyshev nodes with different scaling coefficients $\alpha = 0.1,0.3,0.5,1.0$ for binary classification problems with (a) $10$-dimensional and (b) $100$-dimensional isotropic normal distributions with covariance matrix $ 2\mathbf{I}$, where the means are shifted by 5 units in the first dimension.}
\label{fig:com_alpha_normal}
\end{figure}

Figure~\ref{fig:comp_alpha_beta} compares of the Bayes error estimates for ensemble estimator with Chebyshev nodes with different scaling coefficients $\alpha = 0.1,0.3,0.5,1.0$ for a $3$-class classification problems, where the distributions of each class are $50$-dimensional beta distributions with parameters $(3,1)$, $(3,1.5)$ and $(3,2)$. All of the experiments in Figures \ref{fig:com_alpha_normal} and \ref{fig:comp_alpha_beta} show that the performance of the estimator does not significantly vary for the scaling factor in the range $\alpha\in [0.3,0.5]$ and a good performance can be achieved for the scaling factor $\alpha\in [0.3,0.5]$.

\begin{figure}[!h]
	\centering
	\includegraphics[width=0.65\textwidth]{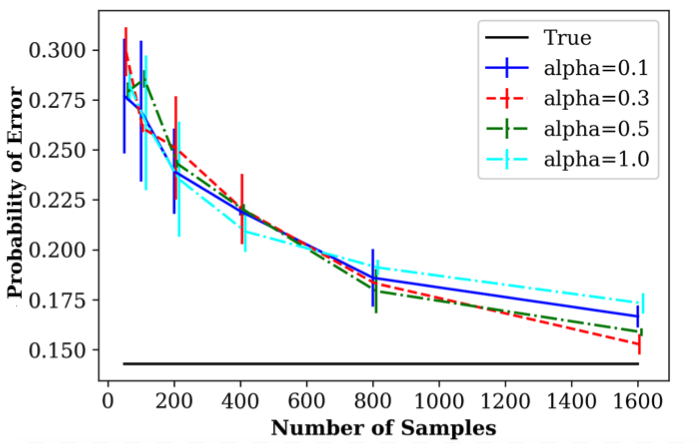}
	\caption{Comparison of the Bayes error estimates for ensemble estimator with Chebyshev nodes with different scaling coefficients $\alpha = 0.1,0.3,0.5,1.0$ for a $3$-class classification problems, where the distributions of each class are $50$-dimensional beta distributions with parameters $(3,1)$, $(3,1.5)$ and $(3,2)$.
}\label{fig:comp_alpha_beta}
\end{figure}

Figure~\ref{fig:cmp_bound_2} compares the optimal benchmark learner with the Bayes error lower and upper bounds using HP-divergence, for
a $3$-class classification problem with $10$-dimensional Rayleigh distributions with parameters $a=2,4,6$.
While the HP-divergence bounds have a large bias, the proposed benchmark learner converges to the true value by increasing sample size.


\begin{figure}
  \centering
  \includegraphics[width=0.7\textwidth]{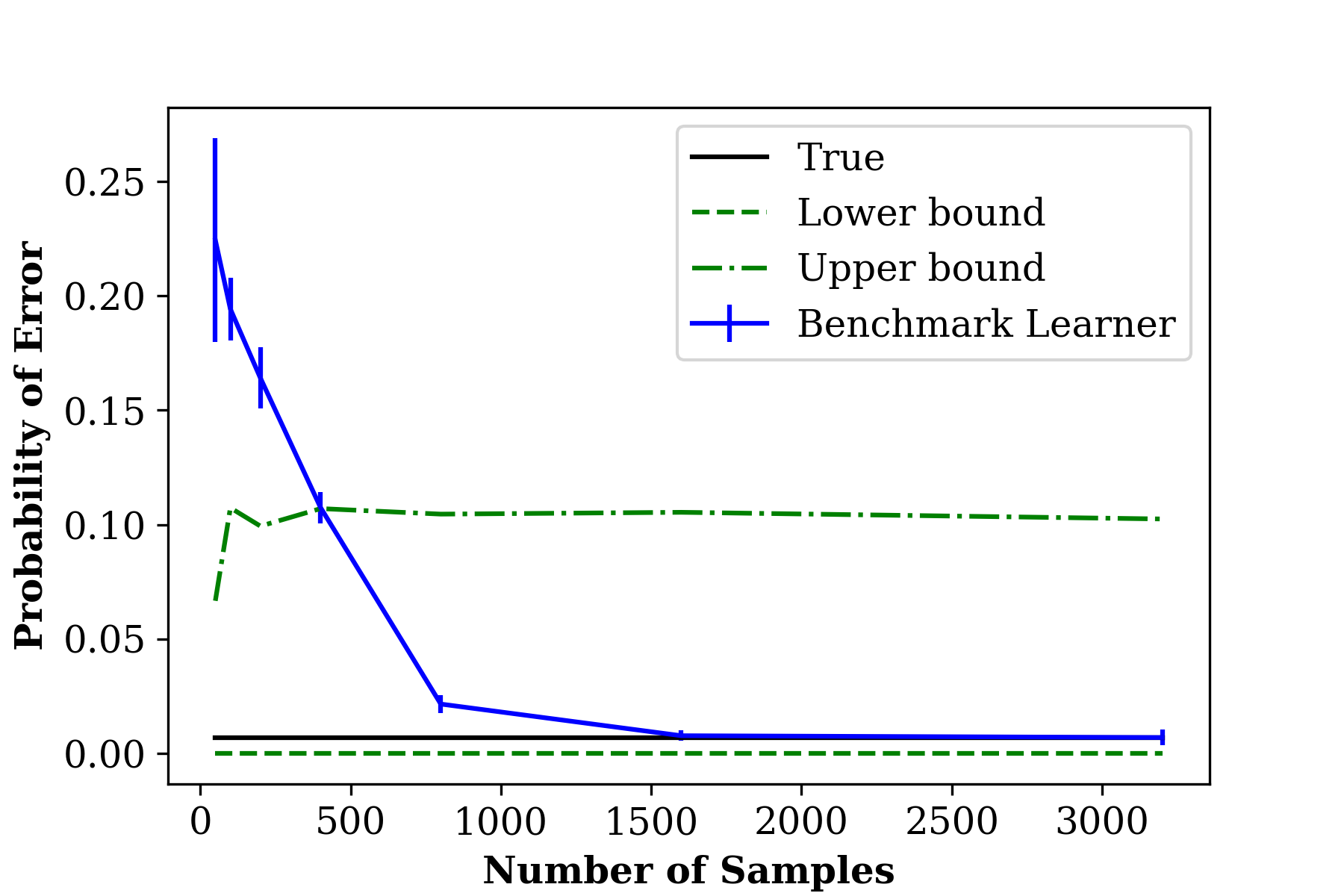}
\caption{Comparison of the optimal benchmark learner (Chebyshev method) with the Bayes error lower and upper bounds using HP-divergence, for a $3$-class classification problem with $10$-dimensional Rayleigh distributions with parameters $a=2,4,6$. While the HP-divergence bounds have a large bias, the proposed benchmark learner converges to the true value by increasing sample size.}
\label{fig:cmp_bound_2}
\end{figure}

In Figure \ref{fig:cmp_clsfr_normal} we compare the optimal benchmark learner (Chebyshev method) with XGBoost and Random Forest classifiers, for a 
$4$-class classification problem $100$-dimensional isotropic mean-shifted Gaussian distributions with identity covariance matrix, where the means are shifted by $5$ units in the first dimension. The benchmark learner predicts the error rate bound better than XGBoost and Random Forest classifiers. 


\begin{figure}
  \centering
  {\centering\includegraphics[width=0.7\textwidth]{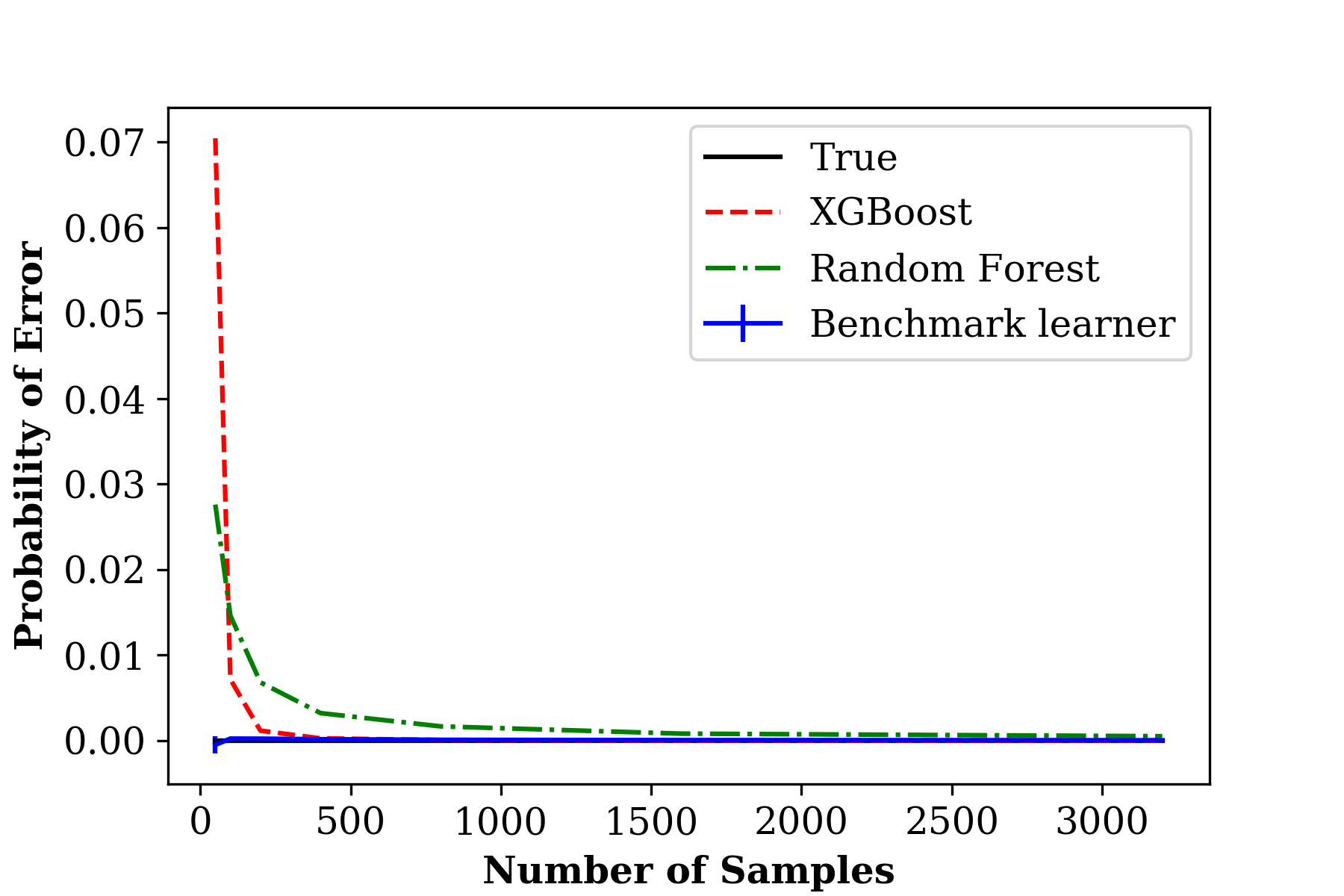}}
\caption{Comparison of the optimal benchmark learner (Chebyshev method) with XGBoost and Random Forest classifiers, for a $4$-class classification problem $100$-dimensional isotropic mean-shifted Gaussian distributions with identity covariance matrix, where the means are shifted by $5$ units in the first dimension. The benchmark learner predicts the Bayes error rate better than XGBoost and Random Forest classifiers. }
\label{fig:cmp_clsfr_normal}
\end{figure}

\bibliography{references}

\end{document}